\documentclass[10pt]{article}

\usepackage[accepted]{icml2026}

\usepackage[
    savespace,
    shorthands,
    version=3,
]{preamble}

\renewcommand{\adri}[1]{}
\renewcommand{\av}[1]{}

\colorlet{eqncolor}{.}
\newcommand{\eqnco}[2][eqncolor]{\begingroup\color{#1}#2\endgroup}

\usepackage{xpatch}
\makeatletter
\xpatchcmd{\proof}{\topsep6\p@\@plus6\p@\relax}{}{}{}
\makeatother

\definecolor{tabgreen}{rgb}{0.17254901960784313, 0.6274509803921569, 0.17254901960784313}
\colorlet{manifoldcolor}{tabgreen}

\let\oldtextbf\textbf
\renewcommand{\textbf}[1]{\oldtextbf{\color{firstcolor} #1}}
\let\figbf\textbf
\let\oldparagraph\paragraph
\renewcommand{\paragraph}[1]{\oldparagraph{\color{firstcolor} #1}}

\newcommand{\parnote}[1]{\textcolor{secondcolor}{\{#1\}}}
\WithSuffix\newcommand\parnote*[1]{} %

\begin{document}
	\twocolumn[
		\icmltitle{An Embarrassingly Simple Way to Optimize Orthogonal Matrices at Scale}

		\begin{icmlauthorlist}
			\icmlauthor{Adri\'an Javaloy}{edi}
			\icmlauthor{Antonio Vergari}{edi}
		\end{icmlauthorlist}
		
		\icmlaffiliation{edi}{Institute for Machine Learning, University of Edinburgh, GB}
		
		\icmlcorrespondingauthor{Adri\'an Javaloy}{ajavaloy@ed.ac.uk}
		\icmlcorrespondingauthor{Antonio Vergari}{avergari@ed.ac.uk}
		
		\icmlkeywords{Machine Learning, ICML}
		
		\vskip 0.3in
	]

	\printAffiliationsAndNotice{}  %
	
    \doparttoc %
	\faketableofcontents %

\begin{abstract}
    Orthogonality constraints are ubiquitous  in  robust and probabilistic machine learning. %
    Unfortunately, current optimizers are computationally expensive and do not scale to problems with hundreds or thousands of constraints.
    One notable exception is the Landing algorithm \citep{ablin2024infeasible} which, however comes at the expense of temporarily relaxing orthogonality.
	In this work, we revisit and improve on the ideas behind Landing, %
    enabling the inclusion of modern adaptive optimizers while ensuring that orthogonal constraints are effectively met. %
    Remarkably, these improvements come at little to no cost, 
    and reduce the number of required hyperparemeters. %
    Our algorithm \ours %
    is fast and  GPU-friendly, \emph{consisting of only \num{5} matrix products}, and in practice maintains orthogonality at all times. %
    On several challenging benchmarks, \ours greatly outperforms recent optimizers and shows it can optimize  problems with thousands of orthogonal matrices in minutes while alternatives would take hours.
    As such, \ours sets a milestone  to finally exploit orthogonality constraints in ML at scale. %
    A public PyTorch implementation of \ours is available at \hbox{\href{https://github.com/adrianjav/pogo}{github.com/adrianjav/pogo}}.
\end{abstract}

\FloatBarrier

\section{Introduction} \label{sec:intro}

\begin{figure}[t]
    \centering
    \includegraphics[width=.95\linewidth]{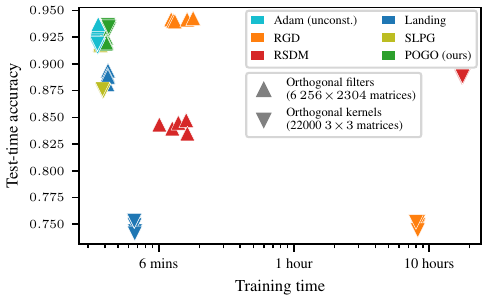}
    \caption{\figbf{\Ours optimizes thousands of orthogonal matrices orders of magnitude faster than retraction methods while achieving performance comparable to unconstrained optimizers}, as shown in this CIFAR-10 \citep{krizhevsky2009learning} classification problem with a tailored CNN \citep{jordan202494} parameterized with orthogonal filters or  kernels. %
    While RSDM \citep{han2025efficient} takes \textbf{17 hours} to train on average, \textbf{\ours trains in 3 minutes}.
    }
    \label{fig:fig1}
\end{figure}

\parnote*{Importance of Stiefel optimization with applications.}

Orthogonality constraints are common in classic machine learning (ML) tasks such as independent and principal component analysis %
\citep{bishop2006pattern,hyvarinen2001independent}.
This is also the case
in modern ML, where orthogonality %
plays a crucial role in ensuring the training stability of large recurrent networks \citep{arjovsky2016unitary,kiani2022projunn}, reducing gradient conflict in multitask learning \citep{javaloy2021rotograd}, or fostering filter diversity in convolutional neural networks (CNNs) \citep{wang2020orthogonal}, vision transformers (ViTs) \citep{fei2022vit} and enabling efficient squared probabilistic circuits \citep{loconte2025square}.
\parnote*{Why it is difficult to do + Discuss prior works in broad.}%
While %
methods such as Riemannian gradient descent (RGD) \citep{absil2008optimization} are staples for the classic ML tasks mentioned before, they struggle on modern problems that involve not just one orthogonal matrix, but several (and potentially large) matrices---from dozens to thousands---that interact between each other during training \citep{li2019orthogonal,fei2022vit,loconte2025square}, see \cref{fig:fig1}.

\NewDocumentCommand{\desideratum}{m o}{\ac[\IfValueTF{#2}{#2}{\textbf{D#1}}]{d#1}}

This begs the question: How should we design an optimizer for orthogonality constraints, from here on \textit{orthoptimizer} for short, to successfully handle modern large-scale ML problems?
We argue that such an orthoptimizer should be: 
\acdef[d1]{\textbf{D1)}}~{\textbf{\emph{feasible}}}, %
ensuring that the optimized matrices are orthogonal, up to numerical precision;
\acdef[d2]{\textbf{D2)}}~{\textbf{\emph{scalable}}}, %
being able to optimize a large number of matrices %
with minimal overhead \wrt* unconstrained deep learning (DL) optimizers; and
\acdef[d3]{\textbf{D3)}}~{\textbf{\emph{competitive}}}, achieving the best downstream performance possible within the given constraints.
Clearly, obtaining all desiderata \textbf{D1-3} is challenging, as optimizing for one might trade-off the others,
\eg to enhance scalability (\desideratum2), one often has to relax orthogonality (\desideratum1) with the promise to recover it ``at the end'' of optimization.

\parnote*{The Landing family scales but leaves the manifold.}

One remarkable example of the above is %
\textit{Landing} \citep{ablin2022fast} and its variants \citep{ablin2024infeasible,vary2024optimization,song2025distributed,loconte2025square} which, instead of relying on costly retractions in a classic Riemannian fashion, use gradient descent over a modified objective (the \emph{landing field}), 
requiring only cheap matrix multiplications (\desideratum2).
As a result, however,
Landing %
makes some compromises. %
First, orthogonality constraints are relaxed during optimization (\desideratum1), allowing matrices to temporarily violate them until they eventually \emph{land} back to the space of orthogonal matrices.
Yet, as we show in \cref{subsec:exp-nns}, %
if several matrices are simultaneously optimized it is uncertain \emph{when} this landing will happen.
Second, Landing methods rely entirely on SGD-like updates and therefore miss on techniques like adaptive gradients and momentum that modern optimizers such as Adam \citep{kingma2014adam} or Muon \citep{jordan2024muon} bring to DL optimization (\desideratum3).

In fact, when we measure the raw downstream performance (\eg, classification accuracy) of models with orthogonality constraints trained with Landing---but also with other SoTA orthoptimizers such as RSDM \citep{han2025efficient}---they consistently lag behind  unconstrained models trained with unconstrained but adaptive optimizers, as shown in \cref{fig:fig1}.
\textit{This highlights an existing gap between theory and practice}:
On the one hand, %
orthogonality %
should benefit model robustness and stability, as argued above, but on the other hand this promise is never truly met as training constrained models %
is harder, slower and yields subpar performance.

\paragraph{Contributions.} 
In this paper, we propose an orthoptimizer that substantially reduces this gap, as it proves to be as competitive as Adam on challenging DL benchmarks~(\desideratum3) with thousands of orthogonal matrices, while being orders-of-magnitude faster than retraction methods (\desideratum2) and closely respecting orthogonality constraints (\desideratum1), as shown in \cref{fig:distance-cnn}. 
Our \textbf{\textit{\ourslong (\ours)}} %
extends Landing %
in several ways.
First, \ours can %
wrap unconstrained optimizers and leverage ``optimization tricks'' like adaptive momentum (\cref{subsec:add-optims}).
Similar to Landing, \ours %
requires only matrix multiplications and therefore scales gracefully on the GPU.
In contrast to Landing, however, \ours %
is always within an epsilon of the orthogonal manifold under mild conditions, as it %
can compute exactly the optimal step size to land back on the manifold (\cref{subsec:compute-lambda}).
Furthermore, we prove that under well-behaved gradient norms we can approximate said step size by a constant value (\cref{subsec:approximate-lambda}).
Finally, we empirically show (\cref{sec:experiments}) that \ours yields SoTA %
downstream performance, while being orders-of-magnitude faster than retraction methods and closely respecting orthogonality constraints in a number of benchmarks.

\vspace{-.75em}

\section{Optimization on the Stiefel Manifold}

\parnote*{Equation with an opt. problem where we want to optimize a (large) set of Stiefel matrices.}

In this work, we look at the following problem:
\begin{align} \label{eq:problem-statement}
    &\minimize_{\mX_1, \dots, \mX_L} \; \mathcal{L}({\mX_1, \dots, \mX_L}) \\
    &\quad \text{s.t.} \; \mX_\indexone \mX_\indexone^\top = \vI_{p_\indexone} \; \text{for}\; i = \range{L}  \nonumber
\end{align}
where $\mathcal{L}$ is any given differentiable %
function and 
each of the $L$ matrices $\mX_\indexone \in \sR^{p_\indexone \times n_\indexone}$ is a wide matrix ($p_\indexone \leq n_\indexone$) taking values on the real numbers.\footnote{We assume the reals for simplicity, but all derivations can be easily extended to other fields like the complex numbers.} %
In practice, $\mX_\indexone$ can be the attention matrix of a transformer \citep{fei2022vit}, the weights of a normalizing flow \citep{GoliskiImprovingNF} or recurrent network \citep{arjovsky2016unitary}, or the filters of a CNN \citep{wang2020orthogonal}.
As a result, $L$ can go up to hundreds if we, \eg, consider orthogonal filters on a ResNet-110 \citep{bansal2018can} or even several thousands if one uses orthogonal kernels \citep{ozay2016optimization}.

\parnote*{What is a Riemannian manifold in plain terms}

A common way to solve \cref{eq:problem-statement} is to optimize over a Riemannian manifold \citep{absil2008optimization}. 
\parnote*{What is our Stiefel manifold (and our convention!)}
In particular, we denote by $\Stiefel{p}{n}$ the manifold of \emph{row-orthogonal} %
matrices, with $p \leq n$, %
known as the Stiefel manifold \citep{stiefel1935richtungsfelder}, %
\begin{equation}
    \Stiefel{p}{n} \coloneqq \left\{ \mX \in \sR^{p \times n} \mid \mX\mX^\top = \mI_p \right\} \eqp{,}
\end{equation}
and consider here for simplicity the Euclidean metric %
on the Stiefel manifold, \ie, the metric induced by the Frobenius inner product in the ambient space,
\begin{equation}
    \innerp{\mX}{\mY} \coloneqq \trace(\mY^\top \mX) \quad \text{for}\; \mX, \mY \in \sR^{p \times n} \eqp{.}    
\end{equation}

Given $\mX \in \Stiefel{p}{n}$, we can define its tangent space, \tangent, which %
represents a hyperplane tangent to $\mX$ %
where the (Riemannian) gradients \wrt $\mX$ reside. 
Furthermore, it can be shown that a matrix $\mA \in \tangent$ if and only if it can be written as $\mA = \mX \mS$ where $\mS$ is a skew-symmetric matrix \citep{edelman1998geometry,ablin2024infeasible}.

\parnote*{Explain Euclidean $\rightarrow$ relative $\rightarrow$ Riemannian gradient}

As described by \citet{ablin2022fast}, we can always compute the Riemannian gradient $\grad$ %
from the Euclidean one as follows: \itemi~map $\nabla f(\mX)$ to its relative gradient \citep{cardoso2002equivariant} by $\mS\!\coloneqq\!\Skew(\mX^\top\nabla f(\mX))$, where $\Skew(\mA) = 1/2 (\mA\!-\!\mA^\top)$ is the skew operator; and \itemii~compute the Riemannian gradient %
by left-multiplying $\mS$ with $\mX$, \ie, $\grad = \mX\mS \in \tangent$\eqp{.} %

\parnote*{The exponential map and Riemannian gradient descent}

Finally, the exponential function \hbox{$\operatorname{Exp}_{\mX}: \tangent \rightarrow \Stiefel{p}{n}$} provides a way to project $\mA \in \tangent$ back to the manifold. 
With these tools, we can then define \textit{Riemannian gradient descent} (RGD) \citep{absil2008optimization} 
with learning rate $\eta > 0$ as an iterative algorithm updating $\mX$ as:
\begin{equation}
    \mX_{t+1} = \operatorname{Exp}_{\mX_t}(-\eta \grad[f(\mX_t)]) \eqp{.}
\end{equation}
\parnote*{Riemannian Gradient Descent with retractions}%
Unfortunately, computing the exponential function is expensive and numerically unstable \citep{arioli1996pade,moler-exponential}.
In practice, the exponential function is approximated with \emph{retractions}, \ie, cheaper functions that are locally equivalent to it. %
In the case of the Stiefel manifold, typical retractions are the QR, %
Cayley or polar retractions, see for example \citep[Ex. 4.1.2]{absil2008optimization}.

However, all of these retractions are either numerically unstable (\eg, the Cayley map involves matrix inversions), or rely on iterative processes (\eg, QR or SVD) that also require GPU-CPU communication \citep{dongarra-svd}.
Next, we focus on a more scalable alternative to Riemannian methods: the Landing algorithm.

\subsection{The Landing Algorithm}
\label{subsec:landing}

\parnote*{Introduced by X.}

Recently, \citet{ablin2022fast} introduced Landing as an alternative approach to solve \cref{eq:problem-statement}. Instead
of using retractions, Landing defines a vector field that pulls iterates to the manifold, letting them escape %
and eventually \emph{land} back, hence the name. %
Namely, Landing initializes $\mX_0 \in \Stiefel{p}{n}$ and produces the following sequence of iterates:
\begin{equation} \label{eq:landing}
    \mX_{t+1} \coloneqq \mX_t - \eta \Lambda(\mX_t) \eqp{,}
\end{equation}
where $\Lambda$ represents the \emph{landing field}, defined as
\begin{equation} \label{eq:landing-field}
    \Lambda(\mX) \coloneqq \grad + \lambda \nabla \dist(\mX) \eqp{,}
\end{equation}
where $\nabla \mathcal{N}(\mX)\!=\!(\mX \mX^\top\!-\!\vI_p)\mX$ is the gradient of the squared manifold distance, $\mathcal{N}(\mX) \coloneqq \frac{1}{4} \norm{\mX \mX^\top\!-\!\vI_p}^2$\eqp{.} 
\parnote*{Intuition, orthogonal directions.}%
Crucially, both terms in \cref{eq:landing-field} are \textit{orthogonal} to each other, as depicted in \cref{fig:scheme-landing}, and thus Landing can be interpreted as seamlessly following a loss-informed direction (Riemannian gradient) and a manifold-attractive one (normal gradient).

\begin{figure}[t]
    \centering
    \begin{tikzpicture}[scale=0.45]
        \fill[manifoldcolor!10] (5,-2) -- (7,2) -- (7,2) arc (60:120:14) -- (-7,2) -- (-5,-2) -- (5,-2) arc (60:120:10);
        \fill[white] (5,-2) -- (5,-2) arc (60:120:10) -- (5,-2);
                
        \coordinate [fill=black,inner sep=.8pt,circle,label=180:{$\mX$}] (X) at (-3,2.9);
        \coordinate [label={[yshift=-.5mm]0:{$-\mX\mS$}}] (G) at (0,3.38);
        \coordinate [label=180:{$-\lambda\nabla\mathcal{N}(\mX)$}] (N) at (-2.8,1.7);
        \coordinate [label=0:{$\mX_1$}] (L) at ($(G)-(X)+(N)-(X)+(X)$);
                
        \coordinate [label=0:{$\Stiefel{p}{n}$}] (Stiefel) at (6,0);
        \coordinate [label=180:{\color{manifoldcolor}$\Stiefel{p}{n}[\varepsilon]$}] (eps) at (5,-0.3);
                
        \draw[-{stealth[black]},black,thick] (X) --  (L);
        \draw[-{stealth[maincolor]},maincolor,thick] (X) --  (G);
        \draw[-{stealth[secondcolor]},secondcolor,thick] (X) --  (N);
        \draw[dashed] (G) -- (L);
        \draw[dashed] (N) -- (L);
                
        \draw (6,0) arc (60:120:12);
        \draw[dashed] (7,2) arc (60:120:14);
        \draw[dashed] (5,-2) arc (60:120:10);
    \end{tikzpicture}
    \caption{Illustration of the landing algorithm, adapted from \citep{ablin2024infeasible}. Landing combines two orthogonal gradients at each iteration and adapts the learning rate $\eta$ to keep $\mX_1$ within $\epsilon$-distance from the Stiefel manifold.}
    \label{fig:scheme-landing}
\end{figure}
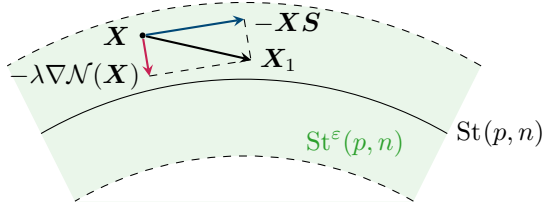

\parnote*{Good. Cheap to compute and sound.}

As a result, Landing is cheap and GPU-friendly, satisfying \desideratum2, as it \emph{only requires matrix multiplications}.
\citet{ablin2024infeasible} also proved local convergence within the manifold for Landing, conditionally on a bound of the step size $\eta$ which ensures that no iteration is farther than $\epsilon$-distance from the manifold. %
Note that $\eps$ needs to be provided as a hyperparameter which defaults to \num{0.5} otherwise. %

\parnote*{Bad. Leaving the manifold (coming back in practice?). Hyperparameters. SGD.}

Unfortunately, despite promises of an eventual landing, it is uncertain in practice \emph{when exactly} this will occur, as it depends on the training dynamics and the hyperparameters provided to the algorithm (see, \eg, \cref{fig:vit}).
Similarly, at each iteration the learning rate is computed based on the current point and gradient \citep[Lem. 3]{ablin2024infeasible}, and the following \emph{fixed} hyperparameters: \itemi~suggested learning rate, $\eta_0$; \itemii~manifold attraction strength, $\lambda$; and \itemiii~maximum manifold distance, $\epsilon$.
This extra complexity results in a trade-off \wrt the gained speed-up, hinders the principled use of typical machine learning techniques such as early stopping \citep{morgan1989generalization} and can compromise downstream performance (\desideratum3).
As we show next, this does \textit{not} need to be the case.

\section{Proximal One-step Geometric Optimization}
\label{sec:methodology}

\begin{figure}[t]
    \centering
    \begin{tikzpicture}[scale=0.45]
        \fill[manifoldcolor!10] (5,-2) -- (7,2) -- (7,2) arc (60:120:14) -- (-7,2) -- (-5,-2) -- (5,-2) arc (60:120:10);
        \fill[white] (5,-2) -- (5,-2) arc (60:120:10) -- (5,-2);
                
        \coordinate [fill=black,inner sep=.8pt,circle,label=180:{$\mX$}] (X) at (-3,2.9);
        \coordinate [label={[yshift=-1mm]0:{$\mM$}}] (G) at (0,3.38);
        \coordinate [label={[yshift=1.5mm]0:{$-\lambda\nabla\mathcal{N}(\mM)$}}] (N) at (0,1.6);
        \coordinate [label=-90:{$\mX_1$}] (L) at ($(N)$);
                
        \coordinate [label=0:{$\Stiefel{p}{n}$}] (Stiefel) at (6,0);
        \coordinate [label=180:{\color{manifoldcolor}$\Stiefel{p}{n}[\varepsilon]$}] (eps) at (5,-0.3);
                
        \draw[-{stealth[black]},black,thick] (X) --  (L);
        \draw[-{stealth[maincolor]},maincolor,thick] (X) --  (G);
        \draw[-{stealth[secondcolor]},secondcolor,thick] (G) --  (N);
                
        \draw (6,0) arc (60:120:12);
        \draw[dashed] (7,2) arc (60:120:14);
        \draw[dashed] (5,-2) arc (60:120:10);
    \end{tikzpicture}
    \caption{Illustration of the \ours algorithm, see \cref{sec:methodology}. Computing the distance \wrt the intermediate point $\mM$, \ours can calculate the exact $\lambda$ to stay within the manifold.}
    \label{fig:scheme-ours}
\end{figure}
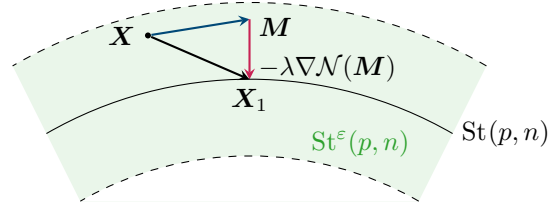

\parnote*{What we do here}

In this section, we dissect the update rule of Landing (\cref{eq:landing-field}) and reveal that it already contains the ingredients for an orthoptimizer that is more feasible~(\desideratum1) and competitive~(\desideratum3) while still being scalable~(\desideratum2), leading to our
\textbf{\textit{\ourslong (\ours)}}.
We provide all proofs for this section in \cref{app:sec:proofs}.

\subsection{Adopting Modern Optimizers}
\label{subsec:add-optims}

We argue that one of the reasons holding Landing from being competitive (\desideratum3)
\wrt* unconstrained optimizers is its reliance on SGD-like updates, as we can easily observe by unrolling \cref{eq:landing,eq:landing-field}:
\begin{equation} \label{eq:unfolded-landing}
    \mX_{t+1} = \mX_t \eqnco{- \eta\grad[f(\mX_t)]} - \eta\lambda \nabla\!\dist(\mX_t) \eqp{.}
\end{equation}
However, note that everything introduced in \cref{subsec:landing} as well as the theory provided by \citet{ablin2024infeasible} holds \emph{as long as $\mS$ is skew-symmetric}, and thus we can safely replace $\nabla f(\mX)$ in $\mS$ by the output of an unconstrained base optimizer (BO), \ie*, by $\mG \coloneqq \BO(\nabla f(\mX)) \eqp{.}$
\parnote*{Using linear optimizers?}%
Now, since $\mX\mS = \mX\Skew(\mX^\top\mG)\mX$ lies in $\tangent$ %
for any $\mG$, %
it is worth asking what properties a base optimizer must have to make $\mX\mS$ a sensible direction to follow.
In this work, we argue that it should be \textit{linear} in the following sense:
\begin{definition} \label{def:linear-optimizer}
	An optimizer is linear if its output $\mG$ is linear \wrt its input $\nabla f(\mX)$, up to a possibly $\nabla f(\mX)$-dependent scalar factor, \ie, $\mG \propto \mA \nabla f(\mX)$ for a matrix $\mA$.
\end{definition}

Then, since the relative gradient is linear \wrt $\mG$, applying a linear $\BO$ becomes an equivariant operation, \ie, 
\begin{align} 
    \mS_{\BO} &= \Skew(\mX^\top\BO(\nabla f(\mX))) \\ &\propto \BO(\Skew(\mX^\top\nabla f(\mX)) = \BO(\mS_{\nabla f(\mX)})\eqp{.}    
\end{align}
In other words, \textit{linear optimizers preserve their semantics}. If we have a linear optimizer, %
it is equivalent applying it in the Euclidean or tangent space, up to scaling.
The vanilla SGD used in Landing, $\mG = \nabla f(\mX)$, is trivially linear,
unlike the popular Adam optimizer, due to its element-wise gradient normalization \citep{kingma2014adam}.
In contrast, Vector Adam (VAdam) \citep{ling2022vectoradam} satisfies \cref{def:linear-optimizer} as it replaces Adam's normalization with a vector-wise one, see \cref{app:sec:vadam}.
As we show in \cref{subsec:osos}, linear base optimizers that are adaptive, such as VAdam, provide a needed competitive edge in complex optimization landscapes. %

\subsection{Leap, Land, Repeat}
\label{subsec:compute-lambda}

\parnote*{Motivation: 1-step distance}

Another aspect to improve in Landing is feasibility (\desideratum1), as it is only asymptotically guaranteed.
Here, we question whether we can land after each iteration or, 
in other words:
\emph{If we were to follow \cref{eq:unfolded-landing}, can we compute the optimal value of $\lambda$ that puts us back on the manifold?}

\parnote*{Intermediate step}

\paragraph{Intermediate step.}
Unfortunately, it seems really difficult to answer the question above if we use \cref{eq:unfolded-landing}.
Instead, we propose to %
reformulate the landing update and follow instead the normal direction of the \emph{intermediate update}:
\begin{align} \label{eq:intermediate}
    \mM_t &\coloneqq \mX_t - \eta \mX_t\mS \eqp{,} \\
    \mX_{t+1} &\coloneqq \mM_t - \lambda \nabla \mathcal{N}(\mM_t) \eqp{.} \label{eq:intermediate-2}
\end{align}
As depicted in \cref{fig:scheme-ours}, $\mM_t$ is the intermediate update result of following $\mX_t\mS$ in tangent space.
This reformulation has two immediate benefits: \itemi~we have disentangled the two steps sizes $\eta$ and $\lambda$, as both were entangled in \cref{eq:unfolded-landing}; and \itemii~we can more easily compute the effect that following the normal direction has on $\mX_{t+1}$, as we show next.
    
\parnote*{Compute lambda in closed-form}

\paragraph{Landing polynomial.}
Let $P(\lambda)$ be the squared distance of $\mX_{t+1}$ after following  \cref{eq:intermediate,eq:intermediate-2}, $P(\lambda) \coloneqq 4\dist(\mX_{t+1})$\eqp{.}
Then, $P(\lambda)$ turns out to have a nice symbolic form:
\begin{lemma} \label{lemma:quartic}
    Let $P(\lambda)$ be as defined above, then $P(\lambda)$ is a quartic polynomial \wrt $\lambda$. Specifically, 
    \begin{align*}
        P(\eqnco{\lambda}) 
        & = \innerfro{\mE}{\mE}\eqnco{\lambda^4} + 2\innerfro{\mE}{\mD}\eqnco{\lambda^3} + %
        \\ 
        & \quad + [\innerfro{\mD}{\mD} + 2\innerfro{\mD}{\mE}]\eqnco{\lambda^2} \\
        & \quad + \innerfro{\mD}{\mC} \eqnco{\lambda} + \innerfro{\mC}{\mC} \eqp{,}
    \end{align*}
    where 
    $\mC \coloneqq \mM_t\mM_t^\top\!-\!\mI$, $\mE\coloneq \nabla\!\dist(\mM_t)\nabla\!\dist(\mM_t)^\top$ and $\mD \coloneqq \mM_t\nabla\!\dist(\mM_t)^\top\!+\!\nabla\!\dist(\mM_t)\mM_t^\top$\eqp{.}
\end{lemma}

We name $P(\lambda)$ the \emph{landing polynomial} and, since it is a quartic polynomial, the problem $\min_\lambda P(\lambda) = 0$ has a \emph{known closed-form solution} using radicals. 
Furthermore, each coefficient in $P(\lambda)$ can be computed in $\bigO(p^2n)$. %
Note that all these operations are efficiently handled by GPUs. 

\parnote*{Note about how to choose the root and mention to the complex numbers (where I can just pick the one with the least absolute value).}

\paragraph{Choosing a step size.}
One remaining question is how to pick $\lambda$ out of the four roots of $P(\lambda)$. 
If $\mX$ took values in an algebraically-closed field (\eg, the complex, $\sC$), the answer would be simple: the one with the \emph{smallest norm}, minimizing the distance from $\mM$.
In the case of real values, we opt for the \emph{closest real value to any of the roots}, \ie, $\argmin_{\indexone\in\{1,\dots,4\}} \min_{\lambda_\indexone \in \sR} (\lambda_\indexone - \bar\lambda_\indexone)^2$ where $\bar\lambda_\indexone$ is one root in the algebraic closure of the field. %
Conveniently, %
this equals %
taking the real part of the root with the least imaginary part.%

\subsection{A Surprising Approximation}
\label{subsec:approximate-lambda}

While we can efficiently solve the landing polynomial $P(\lambda)$, %
this still incurs some additional overhead. %
In this section, we investigate whether there exists a good approximation for $\lambda$ that is cheaper to compute.
\parnote*{Case when we are in the manifold.}%
To this end, we %
start by assuming that $\mX$ lies on the manifold, in which case
we can find a simple bound of how far $\mM$ is from $\Stiefel{p}{n}$:
\begin{proposition} \label{prop:bound-intermediate-step}
	Let $\mX \in \Stiefel{p}{n}$ and $\mM = \mX\!-\!\eta \mX\mS$ with $\mS$ skew-symmetric, then
    \begin{equation} \label{eq:bound-intermediate-step}
        \norm{\mM\mM^\top - \mI} \leq \eta^2 \norm{\mS^2} \eqp{.}
    \end{equation}
\end{proposition}

Remarkably, the bound above depends \emph{exclusively} on the learning rate $\eta$ and the relative gradient, $\mS$.
Now, since $\mX\in\Stiefel{p}{n}$ and $\mS$ is the skew-symmetric part of $\mX^\top\mG$, it is clear that $\norm{\mS}$ is upper-bounded by $\norm{\mG}$.
Therefore, we make the following mild assumption: %

\parnote*{Assumptions}

\begin{assumption} \label{ass:bound-G}
    $\norm{\mG}$ is upper bounded by $L$, \ie, $\norm{\mG} \leq L$\eqp{.}
\end{assumption}
Note that the assumption above bounds the norm of $\mG$, \ie, \emph{the output of the base optimizer}, and not the Euclidean gradient, $\nabla f(\mX)$.
Now, if we define $\xi \coloneqq \eta L$, then \cref{ass:bound-G} implies that $\eta^2\norm{\mS^2} \leq \xi^2$ (see \cref{app:prop:bound-trace}).
As it turns out, if we keep $\xi < 1$ by \eg setting the learning rate appropriately, then $\lambda = 1/2$ becomes a great approximation:
\begin{proposition} \label{prop:polynomial-bound-stiefel}
	Let $\mX \in \Stiefel{p}{n}$ and assume that $\xi < 1$, then we have that $P(1/2) = \littleo(\xi^7)$\eqp{.}
\end{proposition}
 
\parnote*{Introduce $\epsilon$-closesness.}%
This result is quite remarkable: If we keep $\norm{\mG}$ under control, %
then we can fix $\lambda$ to $1/2$ and stay within a $\littleo(\xi^{7/2})$-radius from the manifold.
We now move to the general case, where we assume that $\mX$ is at most $\epsilon$-far from the manifold, 
and show that the landing polynomial of $\mX$ can be bound as a function %
of its projection to the manifold:

\parnote*{General case polynomial}

\begin{theorem} \label{prop:general-polynomial-bound}
    Let $\mX \in \sR^{p\times n}$ \st* %
    $\norm{\mX\mX^\top - \mI} \leq \epsilon$
    and further assume that $\xi < 1$. Then
	\begin{align*}
		P(\lambda)
		&\leq 2 P_\mY(\lambda) + 2[2\!+\!2\sqrt{P_\mY(\lambda)}\!+\!Q(\lambda, \epsilon)]^2 Q(\lambda, \epsilon)^2 \eqp{,}
	\end{align*}
	where $P_\mY(\lambda)$ is the landing polynomial of $\mY$, the projection of $\mX$ onto the Stiefel manifold, %
	and where
	\begin{equation}
		Q(\lambda, \epsilon) \coloneqq (24\lambda + 3)\epsilon + \littleo(\epsilon) \quad \text{as } \epsilon \rightarrow 0 \eqp{.}
	\end{equation}
\end{theorem}

Thus, we can bound the landing polynomial (\ie, the squared distance of $\mX_{t+1}$ with a step size of $\lambda$, see \cref{eq:intermediate-2}) as a function of only $\lambda$ and the distance of $\mX_t$. %
\parnote*{General case.}%
If we now start reasoning from $\mX_0 \in \Stiefel{p}{n}$, we know through \cref{prop:polynomial-bound-stiefel} that one iteration of \cref{eq:intermediate,eq:intermediate-2} with $\lambda = 1/2$ puts $\mX_1$ at most $\littleo(\xi^{7/2})$-far from the manifold. 
Then, we can bound the distance of $\mX_2$ using  \cref{prop:general-polynomial-bound} with $\epsilon = \littleo(\xi^{7/2})$ and $\lambda = 1/2$.
Finally, if we continue with this inductive process, we get to the proof of the following result:
\parnote*{Bound for $\lambda = 0.5$}%
\begin{theorem} \label{thm:main-result}
	If $\mX_0 \in \Stiefel{p}{n}$, $\xi \coloneqq \eta L < 1$ and  $\lambda = 1/2$, then for every $\mX_t$ produced by \cref{eq:intermediate,eq:intermediate-2} it holds that
	\begin{equation}
		P(1/2) = \norm{\mX_t \mX_t^\top - \mI}^2 = \littleo(\xi^7) %
        \eqp{.}
	\end{equation}
\end{theorem}
In other words, if we ensure that $\xi < 1$ by a combination of \itemi decreasing $\eta$ and \itemii using base optimizers with gradient normalization, \eg VAdam \citep{ling2022vectoradam}, then \cref{thm:main-result} ensures that all iterations $\mX_t$ from \cref{eq:intermediate,eq:intermediate-2} stay close to the manifold if we simply set $\lambda = 1/2$.

\parnote*{Intuition of the solution $\lambda = 1/2$.}

\paragraph{Intuition.}

We note that by substituting $\lambda = 1/2$ in \cref{eq:intermediate-2} we obtain $\mX_{t+1} = (\frac{3}{2} \mI - \frac{1}{2}\mM\mM^\top)\mM$, which equals the last step of SLPG \citep{liu2024penalty}, as discussed in \cref{app:sec:relation-slpg}.
We can leverage this coincidence to shade some light on why $\lambda = 1/2$ is such a good approximation.
Namely, \citet{liu2024penalty} proposed to use the expression above as an approximation to a polar retraction, $(\mM\mM^\top)^{-1/2}\mM$, since $\frac{3}{2} - \frac{1}{2}z$ is a first-order Taylor expansion of $z^{-1/2}$ at 1.
As such, we can interpret \cref{eq:intermediate} ($\mM_t$) as a gradient step in \tangent, and \cref{eq:intermediate-2} ($\mX_{t+1}$) as an approximated polar retraction.
This intuition further supports \cref{prop:polynomial-bound-stiefel,thm:main-result}.

\parnote*{Convergence using landing results.}

\paragraph{Convergence.}

Given the connection of \ours with Landing and the intuition  above, the convergence of \ours to a stationary point where $\norm{\mS} = \norm{\nabla\dist(\mX)} = 0$ could be shown in two ways.
First, we can leverage Landing convergence results by reinterpreting \cref{eq:intermediate,eq:intermediate-2} as Landing with alternating step sizes (such that $\eta\lambda = 0$ and $\eta + \lambda \neq 0$).
Second, we can also interpret \ours for the case $\lambda = 1/2$ and $\mG = \nabla f(\mX)$ as an $\littleo(\xi^{7/2})$-approximation of RGD with polar retractions, whose convergence is well-understood, \eg in \citet[Cor. 4.9]{boumal2023introduction}.

\subsection{Summary and Practical Considerations}

\parnote*{Summary.}

To recapitulate, we provide in \cref{alg:ours} a description of what constitutes \ours, which extends the ideas behind Landing such that, at every step, it approximately finds a point on the manifold proximal to the intermediate update.

\parnote*{Computational cost.}

\paragraph{Computational cost.} 
As described in \cref{alg:ours}, \ours only consists of additions and multiplications, and is therefore GPU-friendly. 
Namely, \ours needs \emph{5 matrix multiplications} if we set $\lambda = 1/2$, $\bigO(p^2 n)$, and an additional $\bigO(p^4 n)$ otherwise.
While \Ours requires one more matrix multiplication than Landing,  it also removes the need to compute a ``step-size safeguard'' to keep $\mX_t$ $\epsilon$-close to the manifold, and reduces hyperparameters to a minimum: base optimizer, learning rate, and whether to fix or compute~$\lambda$.
\parnote*{High precision and floats.}%
Moreover, \ours enjoys three more subtle advantages over retraction-based methods: \itemi~since it uses only matrix sums and products, it is numerically stable; \itemii~it benefits from speed-ups in matrix-multiplication primitives such as the reduction of mantissa bits \citep{henry2019leveraging}; and, most remarkably, \itemiii~it circumvents the use of \emph{batch SVD}, which is notoriously difficult to parellize on GPUs \citep{abdelfattah2026efficient}, explaining the significant speed-ups observed in \cref{fig:fig1}.

\parnote*{Other manifolds}

\paragraph{Other manifolds.} 
One natural question is whether \ours can be extended to manifolds other than the Stiefel manifold.
As we demonstrate in \cref{subsec:osos}, this is the case for the complex Stiefel manifold, where $\mX \in \sC^{p \times n}$ and transposes become adjoint matrices.
In addition, $\Stiefel{n}{n}$ and $\Stiefel{n-1}{n}$ are diffeomorphic to the orthogonal and special orthogonal groups, respectively.
Other cases require careful consideration for the manifold field %
and potential %
functions, as also discussed by \citet[\S3.5]{ablin2022fast}.

\begin{algorithm}[t]
    \small
    \caption{\ourslong.}%
    \label{alg:ours}
    \textbf{Input:} Input matrix $\mX$, Euclidean gradient $\nabla f(\mX)$ \\
    \textbf{Parameters:} Learning rate $\eta$, $\operatorname{find\ root}$ flag, base optimizer $\BO$. %
\begin{algorithmic}[1]
    \State \Let $\vG = \BO(\nabla f(\mX))$ \Comment{SGD, VAdam, \etc.}
    \State \Let $\vG = \mX\operatorname{Skew}(\mX^\top\vG)$ \Comment{Riemannian gradient}
    \State \Let $\vM = \mX - \eta \vG$ \Comment{Intermediate step}
    \If{$\operatorname{find\ root}$}
        \State \Let $\lambda \gets \operatorname{solve}_\lambda P(\lambda) = 0$ \Comment{Quartic polynomial}
    \Else
        \State \Let $\lambda = 1/2$ \Comment{$\lambda = 1/2$}
    \EndIf
    \State \Return $\vM + \lambda(\vI - \vM\vM^\top)\vM$ 
\end{algorithmic}
\end{algorithm}

\section{Related Works}
\label{sec:related-work}

\paragraph{Riemannian optimization.}
Many works try to improve the scalability of RGD \citep{absil2008optimization} without giving up on feasibility (\desideratum1). As such, these methods depend one way or another on the use of retractions whose scalability is limited. %
Among these we can find gradient-based \citep{abrudan2008steepest}, conjugate gradient \citep{abrudan2009conjugate} or trust-region methods \citep{absil2007trust}.
Interestingly, recent works %
improve scalability by optimizing a random submanifold at each step \citep{han2024riemannian,han2025efficient}.

\paragraph{Trivializations.}
An alternative line of research uses retractions as surjective functions to map unconstrained parameters as feasible solutions \citep{lezcano2019cheap,lezcano2019trivializations}.
Thus, these methods can leverage techniques in Euclidean space, \eg Adam \citep{kingma2014adam}.
Unfortunately, numerical and scalability issues inherent from retractions still persist with these approaches.

\begin{figure*}[t]
    \centering
    \includegraphics{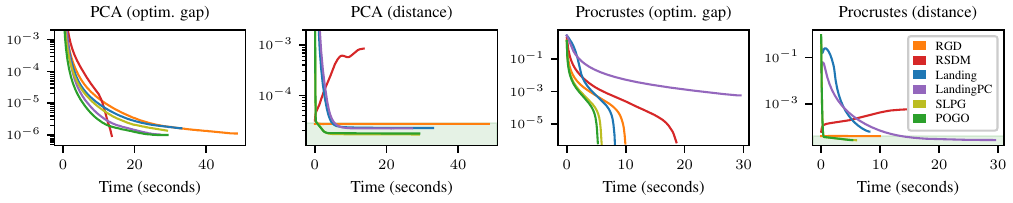}
    \caption{\figbf{\Ours reduces the optimality gap the fastest across all baselines while staying on the manifold}. Results are averaged over \num{10} independent runs and the orthogonal matrices are of size $1500\times 2000$ for PCA and $2000\times 2000$ for Procrustes.} 
    \label{fig:pca-procrustes}
\end{figure*}

\textbf{Infeasible methods} 
trade scalability with feasibility, allowing solutions slightly violate constraints and converge back to feasible solutions.
Among these methods we find Lagrangian-based methods \citep{gao2019parallelizable,gao2022orthogonalization}, as well as Landing and its variants \citep{ablin2022fast,ablin2024infeasible,vary2024optimization,song2025distributed,loconte2025square}.
Finally, we can also find proximal based methods for non-smooth optimization \citep{chen2020proximal,chen2021manifold}. 
In particular, \citet{liu2024penalty} propose a sequential linearized proximal gradient method (SLPG) where at each iteration it applies an approximate
orthogonalization step. 
When they approximate it with a Taylor expansion, their 
normal step coincides with ours when $\lambda=1/2$, and one could recover the \ours update in the case where $p \in \{1, n\}$. 
For an in-depth comparison, refer to \cref{app:sec:relation-slpg}.

\section{\ours in Action}
\label{sec:experiments}

We now evaluate the performance of \ours in relation with existing baselines. 
Namely, we are interested in verifying the three desiderata for an orthoptimizer described in \cref{sec:intro}, \ie: \ac[\textbf{D1)}]{d1}~stay \emph{close to the manifold}; \ac[\textbf{D2)}]{d2}~be computationally efficient; and \ac[\textbf{D3)}]{d3}~obtain a good \emph{downstream performance}.
While here we present the main results, experimental details and additional results can be found in \cref{app:sec:experiments}. 

\parnote*{Baselines.}

\paragraph{Baselines.}
We compare \ours to the following methods: \itemi~Riemannian Gradient Descent (RGD) with QR retraction \citep{absil2008optimization}, Riemannian Random Submanifold Descent (RSDM) with orthogonal sampling \citep{han2025efficient}, Landing \citep{ablin2022fast}, LandingPC \citep{loconte2025square}, and SLPG \citep{liu2024penalty}. 
Unless otherwise stated, we consider \ours with $\lambda$ set to $1/2$ and $\mG = \nabla f(\mX)$, \ie, directly plugging the Euclidean gradient.
Note that some baselines serve as quasi-ablations, \eg, RGD ablates the polar-retraction approximation (\cref{subsec:approximate-lambda}) and SLPG the use of adaptive base optimizers (\cref{subsec:add-optims,app:sec:relation-slpg}).

\subsection{Single-Matrix Optimization}

We focus first on ML problems involving one single orthogonal matrix, conforming a simple testbed where we can study the orthoptimizers in ideal conditions.

To this end, we follow the setup of \citep{han2025efficient} and reproduce the largest experiments therein. Additionally, we do 10 repetitions per experiment and report the average. %

\parnote*{Exp 1 matrix: Linear/Simple problems and comparison with baselines: PCA and Procrustes.}

\paragraph{Online PCA.} We consider first the problem of finding the largest eigenvectors of a given matrix $\mA\in \sR^{n \times n}$, 
\begin{equation} \label{eq:pca-problem}
    \max_{\mX\in\sR^{p\times n}} \; \norm{\mX\mA}^2 \quad \text{such that } \mX\in\Stiefel{p}{n} \eqp{.}
\end{equation}
We set $n = 2000$, $p = 1500$ and, like \citet{han2025efficient}, initialize $\mA\mA^\top$ as a positive definite matrix with a condition number of \num{1000} and exponentially decaying eigenvalues. 
For RSDM we set the submanifold dimension to \num{700}.
As the analytical solution of \cref{eq:pca-problem} is the top-$p$ eigenvectors of $\mA$, we use as performance metric the optimality gap, \ie, the relative error between the optimal loss and that found by the orthoptimizer.
We train all methods by \num{3000} iterations and early stop them if they reach an optimality gap of \num{1e-6}, reporting the training time until then and the distance of $\mX$ to the Stiefel manifold as measured by $\norm{\mX\mX^\top - \mI}$.

We can observe the evolution of the optimality gap and manifold distance on the two left-most plots of \cref{fig:pca-procrustes}. 
There, we see that \ours and LandingPC converge first, with Landing, SLPG and RGD descending at a similar rate. 
Remarkably, RSDM has the slowest start and then sharply descends. %
Regarding feasibility, %
we see every method quickly landing onto the manifold (green area, SLPG overlaps with \ours), except for RSDM which instead seems to %
consistently move away from the manifold. 
\Cref{app:subsec:ablation-pca} shows that the latter is solved  using (slow) 64-bit floating-point arithmetic.

\paragraph{Orthogonal Procrustes problem.} Next, we focus on the problem of matrix alignment \citep{gower2004procrustes}: Given two matrices $\mA \in \sR^{p \times p}$ and $\mB \in \sR^{p \times n}$, solve
\begin{equation}
    \min_{\mX\in\sR^{p\times n}} \; \norm{\mA\mX - \mB}^2 \quad \text{such that } \mX\in\Stiefel{p}{n} \eqp{.}
\end{equation}
We set $p = n = 2000$, initialize all the entries of $\mA$ and $\mB$ with independent standard Gaussian samples, and set the submanifold dimension of RSDM to \num{900}.
Just like before, we know that the analytical solution is the projection onto the Stiefel manifold of the matrix $\mA\mB$, and thus evaluate performance as the optimality gap.
This time, we train each method for a maximum of \num{3000} iterations with early stopping if the optimality gap reaches a value of \num{1e-6}.

The two rightmost plots of \cref{fig:pca-procrustes} show the optimality gap and manifold distance of all methods during training, where we can observe that \ours and SLPG converge significantly quicker, and that LandingPC exhausts all the training iterations.
It is worth noting that the learning rate of LandingPC was set to the maximum allowed to safely bound the manifold distance.
Finally, we observe that \ours and SLPG immediately go to the manifold, whereas both landing methods take considerably longer. Once again, we observe that RSDM consistently strays away from the manifold.

\subsection{Multiple-Matrix Optimization}
\label{subsec:exp-nns}

We just saw that \ours meets our desiderata (\ac[\textbf{D1-3}]{d1}) on experiments with one matrix.
Now we move to cases where orthogonality is imposed to neural-network parameters.
This is ultimately our setting of interest and where we expect the most challenges, as they involve potentially many matrices that interact between them.

\begin{figure}[t]
    \centering
    \includegraphics[width=.93\linewidth]{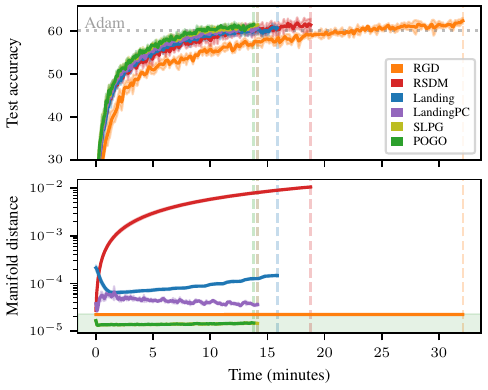}
    \caption{While all methods obtain similar test accuracies with an O-ViT \citep{fei2022vit} on CIFAR-10 \citep{krizhevsky2009learning}, \figbf{\ours is the fastest method to complete \num{10} epochs {without leaving the manifold}}. Results show average and \SI{95}{\percent} confidence intervals over \num{5} independent runs.}
    \label{fig:vit}
\end{figure}

\parnote*{Exp (Scalability): O-Vit.} 

\paragraph{Vision Transformers.}
First, we consider a classification problem on CIFAR-10 \citep{krizhevsky2009learning} using a \SI{28}{\million}-parameter Orthogonal Vision Transformer (O-ViT) \citep{fei2022vit}, a ViT with orthogonal attention matrices.
To this end, we follow the setting of \citet{han2025efficient} and train a small-size O-ViT for \num{10} epochs with all methods, independently repeating the experiment five times.
In this setting, we have \num{18} matrices with $n = p = 1024$ and we set the submanifold of RSDM to $300$ as in the original paper.

\parnote*{O-Vit results}

\Cref{fig:vit} shows the training curves of all methods \wrt test accuracy and manifold distance. %
All methods achieve similar test accuracy, slightly surpassing an unconstrained baseline trained with Adam for reference (gray dotted line).
Thus, the main difference becomes the training speed, where we observe that both landing methods and POGO are twice as fast as RGD, and \ours finishes 5 minutes quicker than RSDM.

We observe bigger discrepancies on the manifold distances. 
Namely, RSDM leaves the manifold again, this time \emph{3 orders of magnitude} further than the reference method, RGD.
Most notably, \ours stays in the manifold all training, as its distance is always lower than that of RGD, the retraction method. 
Interestingly, we see that Landing initially gets closer to the manifold and then leaves it again where its extension, LandingPC, mirrors this behavior: first increases its distance, and then consistently gets closer to the manifold. 

\parnote*{Exp (Scalability): CIFAR10 on ResNet.}

\begin{figure}[t]
    \centering
    \includegraphics[width=.93\linewidth]{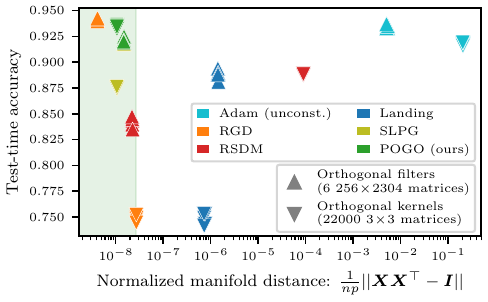}
    \caption{Normalized distance and test-time accuracy on the CNN experiment from \cref{subsec:exp-nns}. We observe that \figbf{in all instances \ours obtains similar accuracy to the unconstrained baseline while staying on the Stiefel manifold}.} \label{fig:distance-cnn}
\end{figure}

\begin{figure}[t]
    \centering
    \includegraphics[width=.93\linewidth]{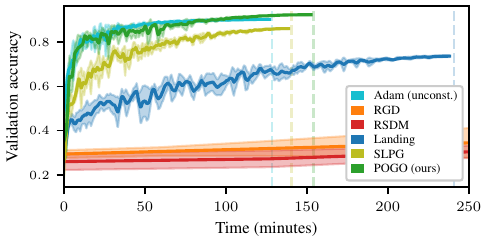}
    \caption{Evolution of the accuracy for the CNN experiment with orthogonal kernels. \figbf{\Ours can learn at the same pace as the unconstrained Adam baseline despite orthogonality constraints.}} \label{fig:cnn-training-curve}
\end{figure}

\paragraph{Convolutional neural networks.}
To test the scalability of \ours (\desideratum2), we now consider a \SI{2}{\million}-parameter CNN tailored for CIFAR-10 classification \citep{jordan202494}. 
Let $O\!\times\!I\!\times\!k\!\times k$ be the parameter dimensions of each convolutional layer, where $O$ and $I$ are the output and input dimensions, %
and $k$ is the kernel size.
Then, we devise two experiments. First, we follow prior works \citep{han2025efficient,ablin2024infeasible} and consider \emph{orthogonal filters} of size $O \times Ik^2$ by flattening the tensor, resulting in \num{6} matrices with sizes ranging from $64\!\times\!216$ up to $256\!\times\!2304$, similar to the O-ViT case. 
Second, we assume instead \emph{orthogonal kernels} \citep{ozay2016optimization} of size $k\!\times\!k$, leading to $OI$ orthogonal matrices per layer and, in total, $218624$ matrices of size $3\!\times\!3$.
In both cases, we train for \num{100} epochs, set the submanifold dimension of RSDM to \num{64} for orthogonal filters, and \num{2} for orthogonal kernels, and use \ours with VAdam.

We present the results in \cref{fig:fig1,fig:distance-cnn}, where  training time is the most remarkable metric: while \ours takes 3 minutes to train in both experiments, RGD and RSDM take double that time with orthogonal filters, and \emph{several hours} more with orthogonal kernels. 
This not only highlights the scalability of \ours, but also the lack of it for methods that rely on the QR algorithm \citep{10.1093/comjnl/4.3.265,kublanovskaya1962some}. 
We attribute the time differences with Landing to implementation details and extra computations for the learning rate.

Most remarkably, we observe that \ours consistently obtains accuracy comparable to Adam with minimal extra overhead (\desideratum3).
Also, while SLPG matches \ours in the orthogonal-filter case (both overlap), it does not match its performance in the orthogonal-kernel case, as we had to train SLPG with very low learning rates to avoid numerical errors. This can be also appreciated in the training curves shown in \cref{fig:cnn-training-curve}.
Finally, as shown in \cref{fig:distance-cnn}, these results are with no compromise in feasibility (\desideratum1): As we make the manifold dimension-invariant, we observe that most methods reach consistent distances (except for RSDM), and that \ours and SLPG stay in every case within the manifold. %

\begin{figure}[t]
    \centering
    \includegraphics{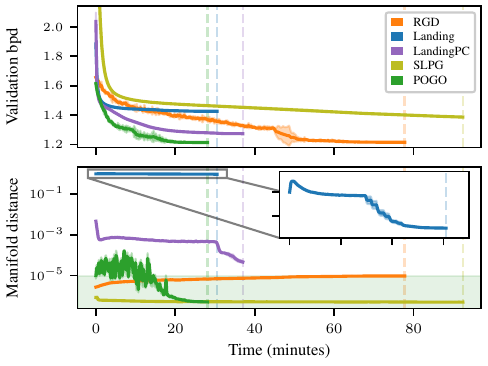}
    \caption{Training curves for validation bits-per-dimension and manifold distance for the squared unitary PC experiment. \figbf{\Ours obtains state-of-the-art results for this class of models, while converging the fastest and not leaving the Stiefel manifold.}} \label{fig:osos}
\end{figure}

\subsection{Large-Scale Orthogonality Constraints for Tractable Inference}
\label{subsec:osos}

In this last section, we introduce a new DL benchmark where \emph{large-scale orthogonality constraints are required}, thus serving as further motivation for orthoptimizers.
An additional ablation study comparing \ours with different learning rates and both policies for $\lambda$ can be found in \cref{app:subsec:ablation}.

\paragraph{Squared unitary PCs.}
Recently, \citet{loconte2025square} introduced a new class of probabilistic circuits (PCs) \citep{choi2020pc,vergari2021compositional} called \textit{squared unitary PCs}.
In summary, squared unitary PCs are expressive models which, by imposing orthogonality constraints to their parameters,  
ensure that one can obtain an already-normalized probability distribution after multiplying the PC with itself. %
For our purposes, squared unitary PCs constitute a great testbed for orthoptimizers not only due to their scalability (up to \SI{300}{\million} parameters, \desideratum2), but also because they represent an example where orthogonality is needed in practice:
Due to prohibitively large memory requirements, using unconstrained parameters and re-normalizing the squared PC is infeasible in practice.

In fact, to train squared unitary PCs \citet{loconte2025square} introduced LandingPC as a variant of Landing tailored for these models, and thus it represents the current SotA.
Moreover, these models can work with complex parameters, serving us an opportunity to test \ours in the complex Stiefel manifold. 
Therefore, we follow an identical setup as \citet{loconte2025square} and simply use each optimizer as drop-in replacements of LandingPC. 

\parnote*{Experimental setup (MNIST, num epochs, early stopping, lr scheduler, etc). Exclude RSDM.}

\paragraph{Experimental setting.} 
We train squared unitary PCs with \num{10} input units to fit the empirical distribution of the MNIST dataset~\citep{lecun2010mnist}, resulting in \num{1048} complex-valued matrices ranging from $10\!\times\!256$ up to $10\!\times\!10000$.
Following \citep{loconte2025square}, every model is trained a maximum of \num{200} epochs, and use early stopping based on the validation set as well as halve the learning rate when the loss plateaus for \num{10} epochs.
Additionally, we remove RSDM from the plots as we did not manage to obtain anything close to the performance of other orthoptimizers.

\parnote*{Results.}

\paragraph{Results.}
\Cref{fig:osos} shows the downstream performance in bits-per-dimension (lower is better) and manifold distance for all methods during training. 
Here, we can observe that RGD obtains quite good performance, although it takes it one hour and a half to do so. 
We can also observe a clear difference between Landing and LandingPC: The latter obtains great performance and consistently nears the manifold, ending nearby \textit{before early stopping kicks in}. In contrast, Landing gets stuck early on at $\epsilon = 0.5$ and the inset plot shows that it spends the rest of training slowly reducing the distance to the manifold. 
Similar to the previous case, we had to set an extremely low learning rate for SLPG not to crash, which explains its slow convergence and great feasibility.
Finally, \ours converges extremely quickly while staying really close to the manifold. \textit{Compared with RGD, \ours obtains the same performance while being  twice as fast}.

\section{Concluding Remarks}

In this work, we substantially advanced the current SoTA of orthoptimizers while preserving three key properties (\hbox{\ac[\textbf{D1-3}]{d1}}).
We have done so by introducing our \ours, which refines the ideas of Landing \citep{ablin2022fast}, resulting in an orthoptimizer which accurately approximates the optimal step size to land on the Stiefel manifold after each gradient update, and thus respecting orthogonality constraints closely. %
\ours raises the bar of how orthoptimizers can scale to real-world deep learning problems with thousands of matrices and of what they achieve in terms of dowstream performance, being finally competitive with unconstrained optimizers.
We foresee that \ours can open up several applications requiring orthogonality constraints, from quantum computing \citep{orus2019tensor} to physics applications \citep{cai2008reduction}, that have been unexplored so far.

\section*{Acknowledgments}

We acknowledge Lorenzo Loconte for his helpful feedback at the first steps of the project, Pierre Ablin for his feedback early in development, and to the \href{https://april-tools.github.io/}{april lab} for their support.
Both authors are supported by the ``UNREAL: Unified Reasoning Layer for Trustworthy ML'' project (EP/Y023838/1) selected by the ERC and funded by UKRI EPSRC.

\section*{Authors Contributions}

AJ conceived the initial idea and discussed it with AV. AJ is responsible for all the theoretical contributions, plots and experiments. AJ led the writing with help from AV who provided feedback at every stage of the project\scalebox{0.01}{ \hphantom{muchmorespace} and sneaked in citations of PCs.}.

\section*{Reproducibility Statement}
All hyperparameter values and experimental settings are detailed in the appendices and properly referenced in the main text.
The code to reproduce the experiments in this paper can be found in \href{https://github.com/april-tools/pogo-experiments}{github.com/april-tools/pogo-experiments}. %

\section*{Impact Statement}

This paper presents work whose goal is to advance the field of machine learning\ optimization. There are many potential societal consequences of our work, none of which we feel must be specifically highlighted here.

    \bibliography{references.clean}

\begin{thebibliography}{54}
\providecommand{\natexlab}[1]{#1}
\providecommand{\url}[1]{\texttt{#1}}
\expandafter\ifx\csname urlstyle\endcsname\relax
  \providecommand{\doi}[1]{doi: #1}\else
  \providecommand{\doi}{doi: \begingroup \urlstyle{rm}\Url}\fi

\bibitem[Abdelfattah \& Fasi(2026)Abdelfattah and
  Fasi]{abdelfattah2026efficient}
Abdelfattah, A. and Fasi, M.
\newblock An {Efficient} {Batch} {Solver} for the {Singular} {Value}
  {Decomposition} on {GPUs}.
\newblock \emph{ArXiv preprint}, abs/2601.17979, 2026.
\newblock URL \url{https://arxiv.org/abs/2601.17979}.

\bibitem[Ablin \& Peyr{\'e}(2022)Ablin and Peyr{\'e}]{ablin2022fast}
Ablin, P. and Peyr{\'e}, G.
\newblock Fast and accurate optimization on the orthogonal manifold without
  retraction.
\newblock In Camps{-}Valls, G., Ruiz, F. J.~R., and Valera, I. (eds.),
  \emph{International {Conference} on {Artificial} {Intelligence} and
  {Statistics}, {AISTATS} 2022, 28-30 {March} 2022, {Virtual} {Event}}, volume
  151 of \emph{Proceedings of Machine Learning Research}, pp.\  5636--5657.
  {PMLR}, 2022.
\newblock URL \url{https://proceedings.mlr.press/v151/ablin22a.html}.

\bibitem[Ablin et~al.(2024)Ablin, Vary, Gao, and Absil]{ablin2024infeasible}
Ablin, P., Vary, S., Gao, B., and Absil, P.-A.
\newblock Infeasible deterministic, stochastic, and variance-reduction
  algorithms for optimization under orthogonality constraints.
\newblock \emph{Journal of Machine Learning Research}, 25\penalty0
  (389):\penalty0 1--38, 2024.

\bibitem[Abrudan et~al.(2009)Abrudan, Eriksson, and
  Koivunen]{abrudan2009conjugate}
Abrudan, T., Eriksson, J., and Koivunen, V.
\newblock Conjugate gradient algorithm for optimization under unitary matrix
  constraint.
\newblock \emph{Signal Processing}, 89\penalty0 (9):\penalty0 1704--1714, 2009.

\bibitem[Abrudan et~al.(2008)Abrudan, Eriksson, and
  Koivunen]{abrudan2008steepest}
Abrudan, T.~E., Eriksson, J., and Koivunen, V.
\newblock Steepest descent algorithms for optimization under unitary matrix
  constraint.
\newblock \emph{IEEE Transactions on Signal Processing}, 56\penalty0
  (3):\penalty0 1134--1147, 2008.

\bibitem[Absil et~al.(2007)Absil, Baker, and Gallivan]{absil2007trust}
Absil, P.-A., Baker, C.~G., and Gallivan, K.~A.
\newblock Trust-region methods on {Riemannian} manifolds.
\newblock \emph{Foundations of Computational Mathematics}, 7\penalty0
  (3):\penalty0 303--330, 2007.

\bibitem[Absil et~al.(2008)Absil, Mahony, and Sepulchre]{absil2008optimization}
Absil, P.-A., Mahony, R., and Sepulchre, R.
\newblock \emph{Optimization algorithms on matrix manifolds}.
\newblock Princeton University Press, 2008.

\bibitem[Arioli et~al.(1996)Arioli, Codenotti, and Fassino]{arioli1996pade}
Arioli, M., Codenotti, B., and Fassino, C.
\newblock The {Pad}{\'e} method for computing the matrix exponential.
\newblock \emph{Linear algebra and its applications}, 240:\penalty0 111--130,
  1996.

\bibitem[Arjovsky et~al.(2016)Arjovsky, Shah, and Bengio]{arjovsky2016unitary}
Arjovsky, M., Shah, A., and Bengio, Y.
\newblock Unitary {Evolution} {Recurrent} {Neural} {Networks}.
\newblock In Balcan, M. and Weinberger, K.~Q. (eds.), \emph{Proceedings of the
  33nd {International} {Conference} on {Machine} {Learning}, {ICML} 2016, {New}
  {York} {City}, {NY}, {USA}, {June} 19-24, 2016}, volume~48 of \emph{{JMLR}
  Workshop and Conference Proceedings}, pp.\  1120--1128. JMLR.org, 2016.
\newblock URL \url{http://proceedings.mlr.press/v48/arjovsky16.html}.

\bibitem[Bansal et~al.(2018)Bansal, Chen, and Wang]{bansal2018can}
Bansal, N., Chen, X., and Wang, Z.
\newblock Can {We} {Gain} {More} from {Orthogonality} {Regularizations} in
  {Training} {Deep} {Networks}?
\newblock In Bengio, S., Wallach, H.~M., Larochelle, H., Grauman, K.,
  Cesa{-}Bianchi, N., and Garnett, R. (eds.), \emph{Advances in Neural
  Information Processing Systems 31: Annual Conference on Neural Information
  Processing Systems 2018, NeurIPS 2018, December 3-8, 2018, Montr{\'e}al,
  Canada}, pp.\  4266--4276, 2018.
\newblock URL
  \url{https://proceedings.neurips.cc/paper/2018/hash/bf424cb7b0dea050a42b9739eb261a3a-Abstract.html}.

\bibitem[Bishop(2006)]{bishop2006pattern}
Bishop, C.
\newblock \emph{Pattern {Recognition} and {Machine} {Learning}}.
\newblock Springer, 2006.
\newblock URL
  \url{https://www.microsoft.com/en-us/research/publication/pattern-recognition-machine-learning/}.

\bibitem[Boumal(2023)]{boumal2023introduction}
Boumal, N.
\newblock \emph{An introduction to optimization on smooth manifolds}.
\newblock Cambridge University Press, 2023.

\bibitem[Cai \& White(2008)Cai and White]{cai2008reduction}
Cai, L. and White, R.~E.
\newblock Reduction of model order based on proper orthogonal decomposition for
  lithium-ion battery simulations.
\newblock \emph{Journal of the Electrochemical Society}, 156\penalty0
  (3):\penalty0 A154, 2008.

\bibitem[Cardoso \& Laheld(2002)Cardoso and Laheld]{cardoso2002equivariant}
Cardoso, J.-F. and Laheld, B.~H.
\newblock Equivariant adaptive source separation.
\newblock \emph{IEEE Transactions on signal processing}, 44\penalty0
  (12):\penalty0 3017--3030, 2002.

\bibitem[Casado(2019)]{lezcano2019trivializations}
Casado, M.~L.
\newblock Trivializations for {Gradient}-{Based} {Optimization} on {Manifolds}.
\newblock In Wallach, H.~M., Larochelle, H., Beygelzimer, A.,
  d'Alch{\'e}{-}Buc, F., Fox, E.~B., and Garnett, R. (eds.), \emph{Advances in
  {Neural} {Information} {Processing} {Systems} 32: {Annual} {Conference} on
  {Neural} {Information} {Processing} {Systems} 2019, {NeurIPS} 2019,
  {December} 8-14, 2019, {Vancouver}, {BC}, {Canada}}, pp.\  9154--9164, 2019.
\newblock URL
  \url{https://proceedings.neurips.cc/paper/2019/hash/1b33d16fc562464579b7199ca3114982-Abstract.html}.

\bibitem[Casado \& Mart{\'i}nez{-}Rubio(2019)Casado and
  Mart{\'i}nez{-}Rubio]{lezcano2019cheap}
Casado, M.~L. and Mart{\'i}nez{-}Rubio, D.
\newblock Cheap {Orthogonal} {Constraints} in {Neural} {Networks}: {A} {Simple}
  {Parametrization} of the {Orthogonal} and {Unitary} {Group}.
\newblock In Chaudhuri, K. and Salakhutdinov, R. (eds.), \emph{Proceedings of
  the 36th {International} {Conference} on {Machine} {Learning}, {ICML} 2019,
  9-15 {June} 2019, {Long} {Beach}, {California}, {USA}}, volume~97 of
  \emph{Proceedings of Machine Learning Research}, pp.\  3794--3803. {PMLR},
  2019.
\newblock URL \url{http://proceedings.mlr.press/v97/lezcano-casado19a.html}.

\bibitem[Chen et~al.(2020)Chen, Ma, Man-Cho~So, and Zhang]{chen2020proximal}
Chen, S., Ma, S., Man-Cho~So, A., and Zhang, T.
\newblock Proximal gradient method for nonsmooth optimization over the
  {Stiefel} manifold.
\newblock \emph{SIAM Journal on Optimization}, 30\penalty0 (1):\penalty0
  210--239, 2020.

\bibitem[Chen et~al.(2021)Chen, Deng, Ma, and So]{chen2021manifold}
Chen, S., Deng, Z., Ma, S., and So, A. M.-C.
\newblock Manifold proximal point algorithms for dual principal component
  pursuit and orthogonal dictionary learning.
\newblock \emph{IEEE transactions on signal processing}, 69:\penalty0
  4759--4773, 2021.

\bibitem[Choi et~al.(2020)Choi, Vergari, and Van~den Broeck]{choi2020pc}
Choi, Y., Vergari, A., and Van~den Broeck, G.
\newblock Probabilistic {Circuits}: {A} {Unifying} {Framework} for {Tractable}
  {Probabilistic} {Modeling}.
\newblock Technical report, University of California, Los Angeles (UCLA), 2020.

\bibitem[Dongarra et~al.(2018)Dongarra, Gates, Haidar, Kurzak, Luszczek, Tomov,
  and Yamazaki]{dongarra-svd}
Dongarra, J., Gates, M., Haidar, A., Kurzak, J., Luszczek, P., Tomov, S., and
  Yamazaki, I.
\newblock The {Singular} {Value} {Decomposition}: {Anatomy} of {Optimizing} an
  {Algorithm} for {Extreme} {Scale}.
\newblock \emph{SIAM Review}, 60\penalty0 (4):\penalty0 808--865, 2018.
\newblock \doi{10.1137/17M1117732}.
\newblock URL \url{https://doi.org/10.1137/17M1117732}.

\bibitem[Edelman et~al.(1998)Edelman, Arias, and Smith]{edelman1998geometry}
Edelman, A., Arias, T.~A., and Smith, S.~T.
\newblock The geometry of algorithms with orthogonality constraints.
\newblock \emph{SIAM journal on Matrix Analysis and Applications}, 20\penalty0
  (2):\penalty0 303--353, 1998.

\bibitem[Fei et~al.(2022)Fei, Liu, Wei, and Chen]{fei2022vit}
Fei, Y., Liu, Y., Wei, X., and Chen, M.
\newblock O-vit: {Orthogonal} vision transformer.
\newblock \emph{ArXiv preprint}, abs/2201.12133, 2022.
\newblock URL \url{https://arxiv.org/abs/2201.12133}.

\bibitem[Francis(1961)]{10.1093/comjnl/4.3.265}
Francis, J. G.~F.
\newblock The {QR} {Transformation} {A} {Unitary} {Analogue} to the {LR}
  {Transformation}--{Part} 1.
\newblock \emph{The Computer Journal}, 4\penalty0 (3):\penalty0 265--271, 1961.
\newblock ISSN 0010-4620.
\newblock \doi{10.1093/comjnl/4.3.265}.
\newblock URL \url{https://doi.org/10.1093/comjnl/4.3.265}.

\bibitem[Gao et~al.(2019)Gao, Liu, and Yuan]{gao2019parallelizable}
Gao, B., Liu, X., and Yuan, Y.-x.
\newblock Parallelizable algorithms for optimization problems with
  orthogonality constraints.
\newblock \emph{SIAM Journal on Scientific Computing}, 41\penalty0
  (3):\penalty0 A1949--A1983, 2019.

\bibitem[Gao et~al.(2022)Gao, Hu, Kuang, and Liu]{gao2022orthogonalization}
Gao, B., Hu, G., Kuang, Y., and Liu, X.
\newblock An orthogonalization-free parallelizable framework for all-electron
  calculations in density functional theory.
\newblock \emph{SIAM Journal on Scientific Computing}, 44\penalty0
  (3):\penalty0 B723--B745, 2022.

\bibitem[Goli\'{n}ski et~al.(2019)Goli\'{n}ski, Lezcano-Casado, and
  Rainforth]{GoliskiImprovingNF}
Goli\'{n}ski, A., Lezcano-Casado, M., and Rainforth, T.
\newblock Improving {Normalizing} {Flows} via {Better} {Orthogonal}
  {Parameterizations}.
\newblock In \emph{{ICML} {Workshop} on {Invertible} {Neural} {Networks}},
  2019.
\newblock URL \url{https://api.semanticscholar.org/CorpusID:271877803}.

\bibitem[Gower \& Dijksterhuis(2004)Gower and
  Dijksterhuis]{gower2004procrustes}
Gower, J.~C. and Dijksterhuis, G.~B.
\newblock \emph{Procrustes problems}, volume~30.
\newblock Oxford university press, 2004.

\bibitem[Han et~al.(2024)Han, Jawanpuria, and Mishra]{han2024riemannian}
Han, A., Jawanpuria, P., and Mishra, B.
\newblock Riemannian coordinate descent algorithms on matrix manifolds.
\newblock In \emph{Forty-first {International} {Conference} on {Machine}
  {Learning}, {ICML} 2024, {Vienna}, {Austria}, {July} 21-27, 2024}.
  OpenReview.net, 2024.
\newblock URL \url{https://openreview.net/forum?id=bdKaQmrM81}.

\bibitem[Han et~al.(2025)Han, Poirion, and Takeda]{han2025efficient}
Han, A., Poirion, P.-L., and Takeda, A.
\newblock Efficient {Optimization} with {Orthogonality} {Constraint}: a
  {Randomized} {Riemannian} {Submanifold} {Method}.
\newblock In Singh, A., Fazel, M., Hsu, D., Lacoste-Julien, S., Berkenkamp, F.,
  Maharaj, T., Wagstaff, K., and Zhu, J. (eds.), \emph{Proceedings of the 42nd
  {International} {Conference} on {Machine} {Learning}}, volume 267 of
  \emph{Proceedings of Machine Learning Research}, pp.\  21814--21843. PMLR,
  2025.
\newblock URL \url{https://proceedings.mlr.press/v267/han25f.html}.

\bibitem[Henry et~al.(2019)Henry, Tang, and Heinecke]{henry2019leveraging}
Henry, G., Tang, P. T.~P., and Heinecke, A.
\newblock Leveraging the bfloat16 artificial intelligence datatype for
  higher-precision computations.
\newblock In \emph{2019 {IEEE} 26th {Symposium} on {Computer} {Arithmetic}
  ({ARITH})}, pp.\  69--76. IEEE, 2019.

\bibitem[Hyv{\"a}rinen et~al.(2001)Hyv{\"a}rinen, Hurri, and
  Hoyer]{hyvarinen2001independent}
Hyv{\"a}rinen, A., Hurri, J., and Hoyer, P.~O.
\newblock Independent component analysis.
\newblock In \emph{Natural {Image} {Statistics}: {A} {Probabilistic} {Approach}
  to {Early} {Computational} {Vision}}, pp.\  151--175. Springer, 2001.

\bibitem[Javaloy \& Valera(2022)Javaloy and Valera]{javaloy2021rotograd}
Javaloy, A. and Valera, I.
\newblock {RotoGrad}: {Gradient} {Homogenization} in {Multitask} {Learning}.
\newblock In \emph{The {Tenth} {International} {Conference} on {Learning}
  {Representations}, {ICLR} 2022, {Virtual} {Event}, {April} 25-29, 2022}.
  OpenReview.net, 2022.
\newblock URL \url{https://openreview.net/forum?id=T8wHz4rnuGL}.

\bibitem[Jordan(2024)]{jordan202494}
Jordan, K.
\newblock 94\% on {CIFAR}-10 in 3.29 {Seconds} on a {Single} {GPU}.
\newblock \emph{ArXiv preprint}, abs/2404.00498, 2024.
\newblock URL \url{https://arxiv.org/abs/2404.00498}.

\bibitem[Jordan et~al.(2024)Jordan, Jin, Boza, You, Cesista, Newhouse, and
  Bernstein]{jordan2024muon}
Jordan, K., Jin, Y., Boza, V., You, J., Cesista, F., Newhouse, L., and
  Bernstein, J.
\newblock Muon: {An} optimizer for hidden layers in neural networks, 2024.
\newblock URL \url{https://kellerjordan.github.io/posts/muon/}.

\bibitem[Kiani et~al.(2022)Kiani, Balestriero, LeCun, and
  Lloyd]{kiani2022projunn}
Kiani, B.~T., Balestriero, R., LeCun, Y., and Lloyd, S.
\newblock {projUNN}: efficient method for training deep networks with unitary
  matrices.
\newblock In Koyejo, S., Mohamed, S., Agarwal, A., Belgrave, D., Cho, K., and
  Oh, A. (eds.), \emph{Advances in {Neural} {Information} {Processing}
  {Systems} 35: {Annual} {Conference} on {Neural} {Information} {Processing}
  {Systems} 2022, {NeurIPS} 2022, {New} {Orleans}, {LA}, {USA}, {November} 28 -
  {December} 9, 2022}, 2022.
\newblock URL
  \url{http://papers.nips.cc/paper\_files/paper/2022/hash/5d1a0188e18c1d74a0f8d6eb5ecede4f-Abstract-Conference.html}.

\bibitem[Kingma \& Ba(2015)Kingma and Ba]{kingma2014adam}
Kingma, D.~P. and Ba, J.
\newblock Adam: {A} {Method} for {Stochastic} {Optimization}.
\newblock In Bengio, Y. and LeCun, Y. (eds.), \emph{3rd {International}
  {Conference} on {Learning} {Representations}, {ICLR} 2015, {San} {Diego},
  {CA}, {USA}, {May} 7-9, 2015, {Conference} {Track} {Proceedings}}, 2015.
\newblock URL \url{http://arxiv.org/abs/1412.6980}.

\bibitem[Krizhevsky et~al.(2009)Krizhevsky, Hinton,
  et~al.]{krizhevsky2009learning}
Krizhevsky, A., Hinton, G., et~al.
\newblock Learning multiple layers of features from tiny images.
\newblock 2009.

\bibitem[Kublanovskaya(1962)]{kublanovskaya1962some}
Kublanovskaya, V.~N.
\newblock On some algorithms for the solution of the complete eigenvalue
  problem.
\newblock \emph{USSR Computational Mathematics and Mathematical Physics},
  1\penalty0 (3):\penalty0 637--657, 1962.

\bibitem[LeCun et~al.(2010)LeCun, Cortes, and Burges]{lecun2010mnist}
LeCun, Y., Cortes, C., and Burges, C.
\newblock {MNIST handwritten digit database}.
\newblock \emph{ATT Labs [Online]. Available:
  http://yann.lecun.com/exdb/mnist}, 2, 2010.

\bibitem[Li et~al.(2019)Li, Jia, Wen, Liu, and Tao]{li2019orthogonal}
Li, S., Jia, K., Wen, Y., Liu, T., and Tao, D.
\newblock Orthogonal deep neural networks.
\newblock \emph{IEEE transactions on pattern analysis and machine
  intelligence}, 43\penalty0 (4):\penalty0 1352--1368, 2019.

\bibitem[Ling et~al.(2022)Ling, Sharp, and Jacobson]{ling2022vectoradam}
Ling, S., Sharp, N., and Jacobson, A.
\newblock {VectorAdam} for {Rotation} {Equivariant} {Geometry} {Optimization}.
\newblock In Koyejo, S., Mohamed, S., Agarwal, A., Belgrave, D., Cho, K., and
  Oh, A. (eds.), \emph{Advances in {Neural} {Information} {Processing}
  {Systems} 35: {Annual} {Conference} on {Neural} {Information} {Processing}
  {Systems} 2022, {NeurIPS} 2022, {New} {Orleans}, {LA}, {USA}, {November} 28 -
  {December} 9, 2022}, 2022.
\newblock URL
  \url{http://papers.nips.cc/paper\_files/paper/2022/hash/1a774f3555593986d7d95e4780d9e4f4-Abstract-Conference.html}.

\bibitem[Liu et~al.(2024)Liu, Xiao, and Yuan]{liu2024penalty}
Liu, X., Xiao, N., and Yuan, Y.-x.
\newblock A penalty-free infeasible approach for a class of nonsmooth
  optimization problems over the {Stiefel} manifold.
\newblock \emph{Journal of Scientific Computing}, 99\penalty0 (2):\penalty0 30,
  2024.

\bibitem[Loconte et~al.(2024)Loconte, Sladek, Mengel, Trapp, Solin, Gillis, and
  Vergari]{loconte2024subtractive}
Loconte, L., Sladek, A.~M., Mengel, S., Trapp, M., Solin, A., Gillis, N., and
  Vergari, A.
\newblock Subtractive {Mixture} {Models} via {Squaring}: {Representation} and
  {Learning}.
\newblock In \emph{The {Twelfth} {International} {Conference} on {Learning}
  {Representations}, {ICLR} 2024, {Vienna}, {Austria}, {May} 7-11, 2024}.
  OpenReview.net, 2024.
\newblock URL \url{https://openreview.net/forum?id=xIHi5nxu9P}.

\bibitem[Loconte et~al.(2025)Loconte, Mengel, and Vergari]{loconte2025sum}
Loconte, L., Mengel, S., and Vergari, A.
\newblock Sum of squares circuits.
\newblock In \emph{Proceedings of the {AAAI} {Conference} on {Artificial}
  {Intelligence}}, volume~39, pp.\  19077--19085, 2025.

\bibitem[Loconte et~al.(2026)Loconte, Javaloy, and Vergari]{loconte2025square}
Loconte, L., Javaloy, A., and Vergari, A.
\newblock How to {Square} {Tensor} {Networks} and {Circuits} {Without}
  {Squaring} {Them}.
\newblock In \emph{The {Fourteenth} {International} {Conference} on {Learning}
  {Representations}}, 2026.
\newblock URL \url{https://openreview.net/forum?id=gHPRSPxIsk}.

\bibitem[Moler \& Loan(2006)Moler and Loan]{moler-exponential}
Moler, C. and Loan, C.
\newblock Nineteen {Dubious} {Ways} to {Compute} the {Exponential} of a
  {Matrix}, {Twenty}-{Five} {Years} {Later}.
\newblock \emph{SIAM Review}, 45:\penalty0 3--49, 2006.
\newblock \doi{10.1137/S00361445024180}.

\bibitem[Morgan \& Bourlard(1989)Morgan and Bourlard]{morgan1989generalization}
Morgan, N. and Bourlard, H.
\newblock Generalization and parameter estimation in feedforward nets: {Some}
  experiments.
\newblock \emph{Advances in neural information processing systems}, 2, 1989.

\bibitem[Or{\'u}s(2019)]{orus2019tensor}
Or{\'u}s, R.
\newblock Tensor networks for complex quantum systems.
\newblock \emph{Nature Reviews Physics}, 1\penalty0 (9):\penalty0 538--550,
  2019.

\bibitem[Ozay \& Okatani(2016)Ozay and Okatani]{ozay2016optimization}
Ozay, M. and Okatani, T.
\newblock Optimization on submanifolds of convolution kernels in cnns.
\newblock \emph{ArXiv preprint}, abs/1610.07008, 2016.
\newblock URL \url{https://arxiv.org/abs/1610.07008}.

\bibitem[Song et~al.(2025)Song, Li, Gao, and Yuan]{song2025distributed}
Song, Y., Li, P., Gao, B., and Yuan, K.
\newblock Distributed {Retraction}-{Free} and {Communication}-{Efficient}
  {Optimization} on the {Stiefel} {Manifold}.
\newblock \emph{ArXiv preprint}, abs/2506.02879, 2025.
\newblock URL \url{https://arxiv.org/abs/2506.02879}.

\bibitem[Stiefel(1935)]{stiefel1935richtungsfelder}
Stiefel, E.
\newblock \emph{Richtungsfelder und {Fernparallelismus} in n-dimensionalen
  {Mannigfaltigkeiten}}.
\newblock PhD thesis, ETH Zurich, 1935.

\bibitem[Vary et~al.(2024)Vary, Ablin, Gao, and Absil]{vary2024optimization}
Vary, S., Ablin, P., Gao, B., and Absil, P.
\newblock Optimization without {Retraction} on the {Random} {Generalized}
  {Stiefel} {Manifold}.
\newblock In \emph{Forty-first {International} {Conference} on {Machine}
  {Learning}, {ICML} 2024, {Vienna}, {Austria}, {July} 21-27, 2024}.
  OpenReview.net, 2024.
\newblock URL \url{https://openreview.net/forum?id=QLtxj3erlJ}.

\bibitem[Vergari et~al.(2021)Vergari, Choi, Liu, Teso, and den
  Broeck]{vergari2021compositional}
Vergari, A., Choi, Y., Liu, A., Teso, S., and den Broeck, G.~V.
\newblock A {Compositional} {Atlas} of {Tractable} {Circuit} {Operations} for
  {Probabilistic} {Inference}.
\newblock In Ranzato, M., Beygelzimer, A., Dauphin, Y.~N., Liang, P., and
  Vaughan, J.~W. (eds.), \emph{Advances in {Neural} {Information} {Processing}
  {Systems} 34: {Annual} {Conference} on {Neural} {Information} {Processing}
  {Systems} 2021, {NeurIPS} 2021, {December} 6-14, 2021, virtual}, pp.\
  13189--13201, 2021.
\newblock URL
  \url{https://proceedings.neurips.cc/paper/2021/hash/6e01383fd96a17ae51cc3e15447e7533-Abstract.html}.

\bibitem[Wang et~al.(2020)Wang, Chen, Chakraborty, and Yu]{wang2020orthogonal}
Wang, J., Chen, Y., Chakraborty, R., and Yu, S.~X.
\newblock Orthogonal {Convolutional} {Neural} {Networks}.
\newblock In \emph{2020 {IEEE/CVF} {Conference} on {Computer} {Vision} and
  {Pattern} {Recognition}, {CVPR} 2020, {Seattle}, {WA}, {USA}, {June} 13-19,
  2020}, pp.\  11502--11512. {IEEE}, 2020.
\newblock \doi{10.1109/CVPR42600.2020.01152}.
\newblock URL \url{https://doi.org/10.1109/CVPR42600.2020.01152}.

\end{thebibliography}
    \bibliographystyle{icml2026}

    \clearpage
    \appendix
    \crefalias{section}{appendix}
    \crefalias{subsection}{subappendix}
    \onecolumn

    \renewcommand{\partname}{}  %
	\part{Appendix} %
	\parttoc %
	\clearpage

    \counterwithin{table}{section}
    \counterwithin{figure}{section}
    \renewcommand{\thetable}{\thesection.\arabic{table}}
    \renewcommand{\thefigure}{\thesection.\arabic{figure}}

\section{Theoretical Results}
\label{app:sec:proofs}

\subsection{Auxiliary Results}

\begin{lemma} \label{app:prop:eigenvalue-epsilon}
    If $\mX \in \Stiefel{p}{n}[\epsilon]$ then %
    $\sqrt{\max(0, 1-\epsilon/\sqrt{p})} \leq \norm{\mX}_2 \leq \sqrt{1+\epsilon/\sqrt{p}}$\eqp{.}
\end{lemma}
\begin{proof}
    Assume $\mX \in \Stiefel{p}{n}[\epsilon]$ and let us denote by $\sigma_\indexone(\mX)$ the \nth{\indexone} singular value of $\mX$. Then, $\norm{\mX\mX^\top - \mI} \leq \epsilon$.
    \begin{align}
        \norm{\mX\mX^\top - \mI}^2 
        &= \norm{\mX\mX^\top}^4 + \norm{\mI}^2 - 2 \innerp{\mX\mX^\top}{\mI} \\
        &= \sum_\indexone \sigma_\indexone(\mX)^2 + \sum_\indexone 1 - 2 \innerfro{\mX\mX^\top}{\mI} \\
        &= \sum_\indexone \sigma_\indexone(\mX)^2 + \sum_\indexone 1 - 2 \sum_\indexone \sigma_\indexone(\mX)^2 \\
        &= \sum_\indexone \left(\sigma_\indexone(\mX)^4 + 1 - 2 \sigma_\indexone(\mX)^2 \right) \\
        &= \sum_\indexone \left(\sigma_\indexone(\mX)^2 - 1\right)^2 \leq \epsilon^2
    \end{align}

    Therefore
    \begin{gather}
        \sum_\indexone \left(\sigma_i(\mX)^2 - 1\right)^2 \leq p \left(\norm{\mX}_2^2 - 1\right)^2 \leq \epsilon^2 \Rightarrow \sqrt{p}\abs{\norm{\mX}_2^2 - 1} \leq \epsilon \\
        -\epsilon/\sqrt{p} \leq \norm{\mX}_2^2 - 1 \leq \epsilon/\sqrt{p} \Rightarrow 1-\epsilon/\sqrt{p} \leq \norm{\mX}_2^2 \leq 1 + \epsilon/\sqrt{p}\\
         \sqrt{\max(0, 1-\epsilon/\sqrt{p})} \leq \norm{\mX}_2 \leq \sqrt{1 + \epsilon/\sqrt{p}}
    \end{gather}
\end{proof}

\begin{proposition} \label{app:prop:rewrite-svd}
    If $\mX \in \Stiefel{p}{n}[\epsilon]$ then we can write $\mX = \mU (\mI + \mDelta) \mV^\top$ where $\norm{\mDelta} \leq \epsilon$ and $\mU, \mV \in \Stiefel{p}{n}$\eqp{.}    
\end{proposition}
\begin{proof}
    Using \cref{app:prop:eigenvalue-epsilon} we know that each singular value of $\mX$ can be written as $\sigma_\indexone(\mX) = \sqrt{1 + \upsilon_\indexone}$ where $\abs{\upsilon_\indexone} \leq \hat\epsilon \coloneqq \epsilon\sqrt{p}$.

    We can rewrite $\sigma_\indexone(\mX)$ as $\sigma_\indexone(\mX) = \pm 1 + \sqrt{1 + \upsilon_\indexone} = 1 + (\sqrt{1 + \upsilon_\indexone} - 1) = 1 + \delta_\indexone$ and bound $\delta_\indexone$:
    \begin{equation}
        \sqrt{\max(0, 1 - \hat\epsilon)} - 1 \leq \delta_\indexone \leq \sqrt{1 + \hat\epsilon} - 1
    \end{equation}
    And, since $\sqrt{\max(0, 1+x)} - 1 \leq x$ if $x \geq 0$ and $x \leq \sqrt{\max(0, 1+x)} - 1$ otherwise:
    \begin{equation}
        -\hat\epsilon \leq \sqrt{\max(0, 1 - \hat\epsilon)} - 1 \leq \delta_\indexone \leq \sqrt{1 + \hat\epsilon} - 1 \leq \hat\epsilon \eqp{.}
    \end{equation}

    We can thus rewrite the singular values of $\mX$ as $\sigma_\indexone(\mX) = 1 + \delta_i$ with $\abs{\delta_i} \leq \hat\epsilon$ and its singular value decomposition as $\mX = \mU (\mI + \mDelta) \mV^\top$ with $\norm{\mDelta} \leq \sqrt{p} \norm{\mDelta}_2 = \hat\epsilon \sqrt{p} = \epsilon$\eqp{.}
\end{proof}

\begin{proposition} \label{app:prop:norm-symmetric}
    For any square matrix $\mX$, we always have that $\norm{\Skew{\mX}} \leq \norm{\mX}$ and $\norm{\Sym{\mX}} \leq \norm{\mX}$\eqp{.}
\end{proposition}
\begin{proof}
    The space of square matrices decomposes as the direct sum of symmetric and skew-symmetric matrices, \ie, $M_p = \operatorname{Sym}_p \oplus \operatorname{Skew}_p$\eqp{.}
    Call $\Sym(\mX) = \frac{1}{2}(\mX + \mX^\top)$ and $\Skew(\mX) = \frac{1}{2}(\mX - \mX^\top)$\eqp{.}
    Then, for any square matrix $\mX$:
    \begin{align}
        \mX &= \frac{2}{2}(\mX \pm \mX^\top) = \frac{1}{2} (\mX + \mX^\top) + \frac{1}{2} (\mX - \mX^\top) = \Sym(\mX) + \Skew(\mX) \eqp{,}
    \end{align}
    and therefore
    \begin{align}
        \norm{\mX}^2 &= \norm{\Sym(\mX)}^2 + \norm{\Skew(\mX)}^2 + \cancelto{0}{2\innerp{\Sym(\mX)}{\Skew(\mX)}} \\
        \norm{\Sym(\mX)}^2 &= \norm{\mX}^2 - \norm{\Skew(\mX)}^2 \leq \norm{\mX}^2 \Rightarrow \norm{\Sym(\mX)} \leq \norm{\mX} \\
        \norm{\Skew(\mX)}^2 &= \norm{\mX}^2 - \norm{\Sym(\mX)}^2 \leq \norm{\mX}^2 \Rightarrow \norm{\Skew(\mX)} \leq \norm{\mX} \eqp{.}
    \end{align}
\end{proof}

\begin{proposition} \label{app:prop:bound-trace}
	Assume that $\mX \in \Stiefel{p}{n}$, $\norm{\mG} \leq L$ and that $\xi \coloneqq \eta L < 1$, then $s_\indexone = \eta^\indexone \trace(\mS^\indexone) \leq \xi^\indexone$.
\end{proposition}
\begin{proof}
	We have the following inequality: 
	\begin{equation}
		\eta^k \norm{\mS^k} \leq \eta^k \norm{\mS}^k \leq \eta^k \norm{\mX^\top\mG}^k = \eta^k \norm{\mG}^k \leq \eta^k L^k = \xi^k \eqp{,}
	\end{equation}
    where $\norm{\mS}\leq\norm{\mX^\top\mG}$ comes from \cref{app:prop:norm-symmetric} and $\norm{\mX^\top\mG}^k = \norm{\mG}^k$ since $\mX \in \Stiefel{p}{n}$.
    
	Then, we can assume $s_\indexone = s_{2k}$ (since $s_\indexone = 0$ for any odd $\indexone$ as $\mS$ is skew-symmetric), and we have that
	\begin{equation}
		s_i = s_{2k} = \eta^{2k} \trace(\mS^{2k}) = \left( \eta^k \norm{\mS^k}\right)^2 \leq \xi^{2k} = \xi^\indexone \eqp{.}
	\end{equation}
\end{proof}

\subsection{Proof of \cref{lemma:quartic}}

\begin{lemma} \label{app:prop:quartic-polynomial}
	Let $\mX \in \sR^{p \times n}$, $\mM = \mX - \eta \mX \mS$ with $\eta > 0$ and $\mS \in \sR^{n \times n}$ skew-symmetric, and $\mX_1 = \mM - \lambda (\mM\mM^\top - \mI)\mM$ with $\lambda > 0$. Then, the function $P(\lambda) = \dist(\mX_1) = \norm{\mX_1\mX_1^\top - \mI}^2$ is a quartic polynomial \wrt $\lambda$.
\end{lemma}
\begin{proof}
    Let us write $\mX_1 = \mA + \lambda \mB$ with $\mA = \mM$ and $\mB = -(\mM\mM^\top - \mI)\mM$\eqp{.} Then, by expanding $P(\lambda)$ we get,
    \begin{align}
        P({\lambda}) 
        &= \norm{(\mA + \lambda \mB)(\mA + \lambda \mB)^\top - \mI_p}^2 \\
        & = \norm{\underbrace{(\mA\mA^\top - \mI_p)}_\mC + \underbrace{(\mA\mB^\top + \mB\mA^\top)}_\mD \lambda + \underbrace{\mB\mB^\top}_\mE \lambda^2}^2
    \end{align}
    Now, let $\mC \coloneqq \mA_t\mA_t^\top - \mI$, $\mD \coloneqq \mA\mB^\top + \mB\mA^\top$ and $\mE\coloneq \mB\mB^\top$\eqp{.}
    Then
    \begin{align}
        P({\lambda}) 
        & = \norm{\mC + \mD\lambda + \mE\lambda^2}^2 \\
        & = \innerfro{\mE}{\mE}{\lambda^4} + 2\innerfro{\mE}{\mD}{\lambda^3}  + [\innerfro{\mD}{\mD} + 2\innerfro{\mD}{\mE}]{\lambda^2}  + \innerfro{\mD}{\mC} {\lambda} + \innerfro{\mC}{\mC} \eqp{.}
    \end{align}
\end{proof}

\subsection{Proof of \cref{prop:bound-intermediate-step}}

\begin{proposition} \label{app:prop:bound-intermediate-step}
	Let $\mX \in \Stiefel{p}{n}$ and $\mM = \mX - \eta \mX\mS$ with $\mS$ skew-symmetric, then
    \begin{equation}
        \norm{\mM\mM^\top - \mI} \leq \eta^2 \norm{\mS^2} \eqp{.}
    \end{equation}
\end{proposition}
\begin{proof}
    First, we simply $\mM\mM^\top - \mI$,
    \begin{align}
        \mM\mM^\top - \mI 
        &= ((\mX - \eta\mX\mS)(\mX - \eta\mX\mS)^\top - \mI) \\
        &= (\cancel{\mX\mX^\top - \mI} - \eta\mX\mS^\top\mX^\top - \eta\mX\mS\mX^\top + \eta^2\mX\mS\mS^\top\mX^\top) \\
        &= (\cancel{\eta\mX\mS\mX^\top - \eta\mX\mS\mX^\top} + \eta^2\mX\mS\mS^\top\mX^\top) \\
        &= -\eta^2\mX\mS^2\mX^\top \eqp{,}
    \end{align}
    so that its squared norm is
    \begin{align}
        \norm{\mM\mM^\top - \mI}^2 
        &= \eta^4\norm{-\mX\mS^2\mX^\top}^2 \\
        &= \eta^4\trace(\mX\mS^2\mX^\top\mX\mS^2\mX^\top) \\
        &= \eta^4\trace(\mS^2\mP\mS^2\mP) \\
        &= \eta^4\norm{\mS^2\mP}^2 \\
        &\leq \eta^4\norm{\mS^2}^2 \eqp{,}
    \end{align}
    where $\mP \coloneqq \mX^\top\mX$ is the orthogonal projector of $\mX$ and thus a bounded operator.
    Therefore,
    \begin{align}
        \norm{\mM\mM^\top - \mI} \leq \eta^2\norm{\mS^2}\eqp{.}
    \end{align}
\end{proof}

\subsection{Proof of \cref{prop:polynomial-bound-stiefel}}

\begin{proposition} \label{app:prop:polynomial-0.5-in-manifold}
	Let $\mX \in \Stiefel{p}{n}$, and let $P(\lambda) = \dist(\mM + \lambda(\mI - \mM\mM^\top)\mM)$ be the landing polynomial of $\mX$, with $\mM = \mX - \eta\mX\mS$, $\eta > 0$ and $\mS = \Skew(\mX^\top \mG)$. Assume also that $\norm{\mG} \leq L$ and let $\xi\coloneqq \eta L$. Then
    \begin{equation}P(1/2) \leq \left(\frac{3}{4} + \frac{1}{4}\xi^2\right)^2 \xi^8 = \littleo(\xi^7) \quad \text{as } \xi \rightarrow 0 \eqp{.}\end{equation} 
\end{proposition}
\begin{proof}
    Using \cref{app:prop:bound-intermediate-step} we have that 
    \begin{align}
        \norm{\mM\mM^\top - \mI}^2 \leq \eta^4\norm{\mS^2}^2 \leq \eta^4 \norm{\mG}^4 \leq (\eta L)^4 = \xi^4 \eqp{.}
    \end{align}

    Then, we can follow similar steps to those of \citet[Lemma 1]{liu2024penalty}:
    \begin{align}
        \norm{\mI - \frac{1}{4}\mM\mM^\top}
        &= \norm{\frac{3}{4}\mI + \frac{1}{4}(\mI - \mM\mM^\top)} \leq \frac{3}{4} + \frac{1}{4}\norm{\mI - \mM\mM^\top} \leq \frac{3}{4} + \frac{1}{4}\xi^2 \eqp{,}
    \end{align}
    which helps us bound $P(\lambda)$ as
    \begin{align}
        \norm{\mX_1\mX_1^\top - \mI}^2
        &= \norm{(\mM + \lambda (\mI - \mM\mM^\top)\mM)(\mM + \lambda (\mI - \mM\mM^\top)\mM)^\top - \mI}^2 \\
        &= \norm{(\mM + \lambda (\mI - \mM\mM^\top)\mM)(\mM^\top + \lambda \mM^\top(\mI - \mM\mM^\top)) - \mI}^2 \\
        &= \norm{(\mM\mM^\top - \mI) + 2\lambda \mM \mM^\top(\mI - \mM\mM^\top) + \lambda^2 \mM \mM^\top(\mI - \mM\mM^\top)^2}^2 \\
        &= \norm{(\mM\mM^\top - \mI) - 2\lambda \mM \mM^\top(\mM\mM^\top - \mI) + \lambda^2 \mM \mM^\top(\mM\mM^\top - \mI)^2}^2 \\
        &= \norm{[\mI - 2\lambda \mM \mM^\top + \lambda^2 \mM \mM^\top(\mM\mM^\top - \mI)](\mM\mM^\top - \mI)}^2 \eqp{,}
    \end{align}
    where, for $\lambda = 1/2$,
    \begin{align}
        \norm{\mX_1\mX_1^\top - \mI}^2 
        &= \norm{[(\mI - \mM \mM^\top) + \frac{1}{4} \mM \mM^\top(\mM\mM^\top - \mI)](\mM\mM^\top - \mI)}^2 \\
        &= \norm{(\mI - \frac{1}{4} \mM \mM^\top)(\mM\mM^\top - \mI)^2}^2 \\
        &\leq \norm{\mI - \frac{1}{4} \mM \mM^\top}^2\norm{(\mM\mM^\top - \mI)^2}^2 \\
        &\leq \left(\frac{3}{4} + \frac{1}{4}\xi^2\right)^2\norm{\mM\mM^\top - \mI}^4 \\
        &\leq \left(\frac{3}{4} + \frac{1}{4}\xi^2\right)^2 \xi^8 = \littleo(\xi^7) \quad \text{as } \xi \rightarrow 0 \eqp{.}
    \end{align}
\end{proof}

\subsection{Proof of \cref{prop:general-polynomial-bound}}

\begin{theorem} \label{app:prop:general-bound}
    Let $\mX \in \Stiefel{p}{n}[\epsilon]$, $\mM = \mX - \eta \mX\mS$ and $\mX_1 = \mM - \lambda (\mM \mM^\top - \mI) \mM$ as defined in \cref{eq:intermediate,eq:intermediate-2}.
    If we assume also that $\norm{\mG} \leq L$ and $\xi \coloneqq \eta L < 1$, then
    \begin{equation}
        \norm{\mX_1 \mX_1^\top - \mI}^2 \leq 2 P_\mY(\lambda) + 2[2 + 2\sqrt{P_\mY(\lambda)} + Q(\lambda, \epsilon)]^2 Q(\lambda, \epsilon)^2
    \end{equation}
    where $P_\mY(\lambda)$ is the landing polynomial of the projection of $\mX$ to the manifold, $\mY = \mU\mV^\top$ where $\mX = \mU\mSigma\mV^\top$ is the singular value decomposition of $\mX$, and where
    \begin{equation}
        Q(\lambda, \epsilon) \coloneqq (24\lambda + 3)\epsilon + \littleo(\epsilon) \quad \text{as } \epsilon \rightarrow 0 \eqp{.}
    \end{equation}
\end{theorem}
\begin{proof}
    Since $\mX \in \Stiefel{p}{n}[\epsilon]$, we can use \cref{app:prop:rewrite-svd} and write $\mX = \mU(\mI + \mDelta)\mV^\top$ where $\norm{\mDelta} \leq \epsilon$.
    Then, if $\mY \coloneqq \mU\mV^\top$ is the projection of $\mX$ onto the Stiefel manifold, $\mX = \mU(\mI + \mDelta)\mV^\top = \mY + \mU\mDelta\mV^\top = \mY + \mB$ where we have that $\norm{\mY} = 1$ and $\norm{\mB} \leq \epsilon$\eqp{.}

    Let us write the Riemannian gradient of $\mX$ as a function of $\mY$ and $\mB$:
    \begin{align}
        \mX \mS 
        &=  \mX \Skew(\mX^\top\mG) \\
        &= (\mY + \mB)\Skew((\mY + \mB)^\top\mG) \\
        &= (\mY + \mB)(\Skew(\mY^\top\mG) + \Skew(\mB^\top\mG)) \\
        &= (\mY + \mB)\Skew(\mY^\top\mG) + (\mY + \mB)\Skew(\mB^\top\mG) \\
        &= \mY\Skew(\mY^\top\mG) + \mB\Skew(\mY^\top\mG) + \mY\Skew(\mB^\top\mG ) + \mB\Skew(\mB^\top\mG) \\
        &= \mY\mS_\mY + \mT_\mR \eqp{,}
    \end{align}
    and thus we can write the Riemannian gradient of $\mX$ as the sum of the Riemannian gradient of $\mY$ plus a remainder term.

    We can also bound each of these terms: 
    \begin{align}
    &\norm{\mB\Skew(\mY^\top\mG)} \leq \norm{\mB}\norm{\Skew(\mY^\top\mG )} \leq \norm{\mB}\norm{\mY^\top \mG} = \norm{\mB}\norm{\mG} \leq L \epsilon  \\
    &\norm{\mY\Skew(\mB^\top\mG)} \leq \norm{\mY}\norm{\mB^\top\mG} = \norm{\mB^\top\mG} \leq  \norm{\mB}\norm{\mG} \leq L \epsilon \\
    &\norm{\mB\Skew(\mB^\top\mG )} \leq \norm{\mB} \norm{\mB^\top\mG} \leq \norm{\mB}^2 \norm{\mG} \leq L \epsilon^2 \\
    & \norm{\mT_\mR} \leq 2 L \epsilon + L \epsilon^2 = (2 + \epsilon) L \epsilon
    \end{align}

    Now we can do something similar with $\mM$;
    \begin{align}
        \mM 
        &= \mX - \eta \mX\mS \\
        &= \mY + \mB - \eta  \mY \mS_\mY - \eta \mT_\mR \\
        &= \mY_1 + \mB - \eta \mT_\mR \\
        &= \mY_1 + \mT_1 \eqp{,}
    \end{align}
    and bound each term:
    \begin{align}
        \norm{\mY_1} 
        &= \norm{\mY - \eta \mY\mS_\mY} \leq \norm{\mY} + \eta \norm{\mY\mS_\mY} = 1 + \eta \norm{\mY\mS_\mY} \\
        & = 1 + \eta\norm{\mS_\mY} \leq 1 + \eta\norm{\mG} \leq 1 + \eta L = 1 + \xi < 2
    \end{align}
    \begin{align}
        \norm{\mT_1} 
        &= \norm{\mB - \eta \mT_\mR} \leq \norm{\mB} + \norm{-\eta \mT_\mR} = \norm{\mB} + \eta \norm{\mT_\mR} \\
        &\leq  \epsilon + \eta(2 + \epsilon) L \epsilon = \epsilon + (2 + \epsilon) \epsilon \xi < (3 + \epsilon) \epsilon
    \end{align}

    We can once again repeat the same process for the distance gradient:
    \begin{align}
        \nabla \dist(\mM) 
        &= (\mM \mM^\top - \mI) \mM \\
        &= ((\mY_1 + \mT_1)(\mY_1^\top + \mT_1^\top) - \mI)\mM \\
        &= (\mY_1(\mY_1^\top + \mT_1^\top) + \mT_1(\mY_1^\top + \mT_1^\top) - \mI)\mM \\
        &= (\mY_1\mY_1^\top + \mY_1\mT_1^\top + \mT_1\mY_1^\top + \mT_1\mT_1^\top - \mI)\mM \\
        &= (\mY_1\mY_1^\top + \mY_1\mT_1^\top + \mT_1\mY_1^\top + \mT_1\mT_1^\top - \mI)(\mY_1 + \mT_1) \\
        &= (\mY_1\mY_1^\top + \mY_1\mT_1^\top + \mT_1\mY_1^\top + \mT_1\mT_1^\top - \mI)\mY_1 + (\mM \mM^\top - \mI)\mT_1 \\
        &= \nabla \dist(\mY_1) + (\mY_1\mT_1^\top + \mT_1\mY_1^\top + \mT_1 \mT_1^\top)\mY_1 + (\mM \mM^\top - \mI)\mT_1 \\
        &= \nabla \dist(\mY_1) + \mT_{\dist}
    \end{align}

    and bound each term:
    \begin{align}
        \norm{(\mY_1 \mT_1^\top + \mT_1 \mY_1^\top)\mY_1} 
        &\leq \norm{2\Sym(\mY_1 \mT_1^\top)\mY_1} \leq 2\norm{\mY_1 \mT_1^\top} \norm{\mY_1} \leq 2\norm{\mY_1}\norm{\mT_1^\top}\norm{\mY_1} \\
        &= 2(1 + \xi)^2 (\epsilon + (2 + \epsilon) \epsilon \xi) < 8 (3 + \epsilon)\epsilon%
    \end{align}
    \begin{align}
        \norm{\mT_1 \mT_1^\top \mY_1} \leq \norm{\mT_1}\norm{\mT_1^\top}\norm{\mY_1} 
        = (1 + \xi) (\epsilon + (2 + \epsilon) \epsilon \xi)^2 < 2(3 + \epsilon)^2\epsilon^2
    \end{align}
    \begin{align}
    \norm{\mM \mM^\top - \mI}
        & = \norm{(\mX - \eta \mX\mS) (\mX - \eta \mX\mS)^\top - \mI)} \\
        & \leq \norm{(\mX \mX^\top - \mI) - \eta \cancel{(\mX \mS^\top\mX^\top + \mX\mS\mX^\top)} + \eta^2 \mX \mS \mS^\top\mX^\top} \\
        & \leq \norm{(\mX \mX^\top - \mI)} + \eta^2\norm{\mX \mS \mS^\top\mX^\top} \\
        & \leq \epsilon + \eta^2\norm{\mX\mS}^2 \\
        & \leq \epsilon + \eta^2\norm{\mX}^2\norm{\mX^\top\mG}^2 \\
        & \leq \epsilon + \xi^2 (1+\epsilon)^2 \\
        & < \epsilon + (1 + \epsilon)^2
    \end{align}
	\begin{align}
		\norm{\mT_{\dist}} 
		& < 8 (3 + \epsilon)\epsilon + 2(3 + \epsilon)^2\epsilon^2 + [\epsilon + (1 + \epsilon)^2](3 + \epsilon) \epsilon \\
		& = 24 \epsilon + \littleo(\epsilon) \quad \text{as } \epsilon \rightarrow 0
	\end{align}

    Finally, we do the same bounding to our quantity of interest:   
    \begin{align}
        \mX_1 
        &= \mM - \lambda (\mM \mM^\top - \mI) \mM \\
        &= \mY_1 + \mT_1 - \lambda \nabla \dist(\mY_1) - \lambda \mT_{\dist} \\
        &= \mY_2 + \mT_1 - \lambda \mT_{\dist} \\
        &= \mY_2 + \mT_2        
    \end{align}
    \begin{align}
        \norm{\mT_2} 
        &\leq \norm{\mT_1} + \lambda\norm{\mT_{\dist}} \\
        &< (3 + \epsilon)\epsilon + \lambda(24 \epsilon + \littleo(\epsilon)) \\
        &= (24\lambda + 3)\epsilon + \littleo(\epsilon) \eqqcolon Q(\lambda, \epsilon)
    \end{align}
    \begin{align}
        \norm{\mX_1 \mX_1^\top - \mI}^2
        &= \norm{(\mY_2 + \mT_2) (\mY_2 + \mT_2)^\top - \mI}^2 \\
        &= \norm{\mY_2 (\mY_2^\top + \mT_2^\top) + \mT_2 (\mY_2^\top + \mT_2^\top) - \mI}^2 \\
        &= \norm{\mY_2 \mY_2^\top + \mY_2 \mT_2^\top + \mT_2 (\mY_2^\top + \mT_2^\top) - \mI}^2 \\
        &= \norm{\mY_2 \mY_2^\top - \mI + \mY_2 \mT_2^\top + \mT_2 (\mY_2^\top + \mT_2^\top)}^2 \\
        &\leq 2 \norm{\mY_2 \mY_2^\top - \mI}^2 + 2\norm{\mY_2 \mT_2^\top + \mT_2 (\mY_2^\top + \mT_2^\top)}^2
    \end{align}
    \begin{align}
        \norm{\mY_2 \mT_2^\top + \mT_2 (\mY_2^\top + \mT_2^\top)}^2
        &= \norm{2 \Sym(\mY_2 \mT_2^\top) + \mT_2\mT_2^\top}^2 \\
        & \leq (\norm{2 \Sym(\mY_2 \mT_2^\top)} + \norm{\mT_2 \mT_2^\top})^2 \\
        & \leq (2\norm{\mY_2}\norm{\mT_2} + \norm{\mT_2}^2)^2 \\
        & = (2\norm{\mY_2} + \norm{\mT_2})^2 \norm{\mT_2}^2 \\
        &\leq [2 (1 + \sqrt{P_\mY(\lambda)}) + Q(\lambda, \epsilon)]^2 Q(\lambda, \epsilon)^2
    \end{align}

    And therefore,
    \begin{align}
        \norm{\mX_1 \mX_1^\top - \mI}^2
        &\leq 2 \norm{\mY_2 \mY_2^\top - \mI}^2 + 2\norm{\mY_2 \mT_2^\top + \mT_2 (\mY_2^\top + \mT_2^\top)}^2 \\
        & \leq 2 P_\mY(\lambda) + 2[2 + 2\sqrt{P_\mY(\lambda)} + Q(\lambda, \epsilon)]^2 Q(\lambda, \epsilon)^2
    \end{align}
\end{proof}

\subsection{Proof of \cref{thm:main-result}}

\begin{theorem}
    If \ours is initialized with $\mX_0 \in \Stiefel{p}{n}$, $\lambda = 1/2$ and $\xi \coloneqq \eta L < 1$, then we have that
    \begin{equation}
        \norm{\mX_t \mX_t^\top - \mI}^2 = \littleo(\xi^7) \quad \text{for every iteration } t \in \sN\eqp{.}
    \end{equation}
\end{theorem}
\begin{proof}
    Say that we start at $\mX_0 \in \Stiefel{p}{n}$ and that $\xi < 1$. Then, $\norm{\mX_1 \mX_1^\top - \mI}^2 = 4P(1/2) = \littleo(\xi^7)$ via \cref{app:prop:polynomial-0.5-in-manifold}.

    So $\mX_1 \in \Stiefel{p}{n}[\epsilon]$ with $\epsilon = \littleo(\xi^{7/2})$\eqp{.} Then, we have for the next iteration $\mX_2$ that
    \begin{equation}
        Q(1/2, \littleo(\xi^3)) = (12 + 3)\littleo(\xi^{7/2}) + \littleo(\littleo(\xi^{7/2})) = \littleo(\xi^{7/2})
    \end{equation}
    and thus
    \begin{align}
        \norm{\mX_2 \mX_2^\top - \mI}^2
        & \leq 2 \littleo(\xi^7) + 2[2 + 2\sqrt{\littleo(\xi^7)} + \littleo(\xi^{7/2})]^2 \littleo(\xi^{7/2})^2 \\
        & = \littleo(\xi^7) + 2[2 + \littleo(\xi^{7/2}) + \littleo(\xi^{7/2})]^2 \littleo(\xi^{7}) \\
        & = \littleo(\xi^7) \eqp{.}
    \end{align}

    Since $\epsilon$ remains $\littleo(\xi^{7/2})$ and the rest of parameters are kept fixed as well ($\lambda = 1/2$, $\eta$ and $L$), then we prove by induction that
    \begin{equation}
        \norm{\mX_t \mX_t^\top - \mI}^2 = \littleo(\xi^7) \eqp{,}
    \end{equation}
    where $t$ is the iteration number, 
    \ie, all the iterations of \ours remain $\littleo(\xi^{7/2})$-close to the manifold as long as $\xi = \eta L < 1$\eqp{.}

\end{proof}

\section{VectorAdam as a Linear Optimizer}
\label{app:sec:vadam}

In this section, we describe why VectorAdam (VAdam) \citep{ling2022vectoradam} meets the definition of a linear optimizer (\cref{def:linear-optimizer}).
First, VAdam computes the vector update as a ratio $\nicefrac{m_{t,i}}{v_{t,i}}$, where the numerator is an exponential moving average (EMA) of the gradient, \ie:
\begin{equation}
	m_t = \text{EMA}_t(\nabla f(\mX_t); \beta_1) = \beta_1 m_{t-1} + (1 - \beta_1)\nabla f(\mX_t).
\end{equation}

And then it it easy to show that the numerator is equivariant \wrt the relative gradient (recall, $\Skew(\mA) = \frac{1}{2}(\mA - \mA^\top)$):
\begin{align}
	\mS_\text{VAdam} &= \mX_t\Skew(\mX_t^\top\text{EMA}_t(\nabla f(\mX_t); \beta_1)) \\
	&= \mX_t \Skew(\mX_t^\top m_t) = \mX_t\frac{1}{2}( \mX_t^\top m_t -  m_t^\top \mX_t) \\
	&= \beta_1 \mX_t \Skew(\mX_t m_{t-1}x) + (1 - \beta_1)\mX_t\Skew(\mX_t^\top\nabla f(\mX_t)) = \\
	&= \text{EMA}_t(\mX_t\Skew(\mX_t^\top \nabla f(\mX_t)); \beta_1) = \text{EMA}_t(\mS_{\nabla f(\mX_t)})
\end{align}

That is, since the relative gradient operator is linear, it is equivalent to compute the numerator of VAdam (in the Euclidean space) and then its relative gradient, than compute the numerator of VAdam for the relative gradient (\ie in the Stiefel manifold).
Finally, the denominator of VAdam is a scalar shared for all parameters, \ie, $v_{t,i} = v_t = ||\nabla f(\mX_t)||^2$. Hence, VAdam is a linear optimizer as in \cref{def:linear-optimizer}.

\section{Relation of \ours with SLPG}
\label{app:sec:relation-slpg}

In this section, we briefly present the original formulation of SLPG \citep{liu2024penalty} as well as under which conditions and the derivations needed to recover something similar to the updates of \ours.

First, note that SLPG was initially introduced by \citet{liu2024penalty} in the context of optimization problems with orthogonality constraints and a non-smooth regularization term. 
Also note that the following formulas use the original formulation by \citet{liu2024penalty} which consider \textit{column-orthogonal} matrices, such that $\mX^\top\mX = \mI_p$.
Within this context, \citet{liu2024penalty} attempt to solve the following problem:
\begin{equation}
    \min_{\mX\in\sR^{n \times p}} f(\mX) + r(\mX) \quad \text{\st}\quad \mX^\top \mX = \mI_p \eqp{.}
\end{equation}

And make the following assumptions (we disregard those from $r$ as they are of no use in our context):
\begin{enumerate}
    \item $f$ is differentiable and locally $L$-smooth.
    \item For any $\mX,\mG \in \sR^{n \times p}$ the problem \begin{equation}
        \min_{\mD \in \sR^{n\times p} } \innerp{\mG}{\mD} + r(\mD) + \frac{1}{2\eta}\norm{\mD - \mX}^2
    \end{equation}
    has a closed-form solution that can be efficiently solved by certain iterative approach.
\end{enumerate}

Then, \citet{liu2024penalty} propose an iterative process that consist in the following steps:
\begin{enumerate}
    \item While the following terminating condition is not met: 
    \begin{equation}
        \norm{\Sym((\mY_k - \mX_k)^\top \mX_k)} \leq c\norm{\mX^\top \mX - \mI} \eqp{,}
    \end{equation}
    \item Approximately solve the following proximal optimization problem:
    \begin{equation}
        \min_{\mD\in\sR^{n\times p}} \innerp{\nabla f(\mX)}{\mD} + r(\mD) + \frac{1}{2\eta}\norm{\mD - \mX_k}^2
    \end{equation}
    with the following iterative process until convergence:
    \begin{enumerate}
        \item Calculate the proximal mapping \begin{equation}\mD_j = \operatorname{prox}_\eta (\nabla f(\mX_k) - \mX_k \Lambda_j; \mX_k) = \argmin_{\mD \in \sR^{n\times p} } \innerp{\mG}{\mD} + r(\mD) + \frac{1}{2\eta}\norm{\mD - \mX}^2\end{equation}
        \item Update $\Lambda_{j+1} = \Lambda_j - \frac{1}{\eta} \Sym((\mD_j - \mX_k)^\top \mX_k)$
    \end{enumerate}
    \item Set $\mY_k = \mD_J$.
    \item Compute an approximate normal step, for which they use a first-order Taylor approximation of a polar retraction, $\mX_{k+1} = \mY_k(\frac{3}{2} \mI_p - \frac{1}{2}\mY_k^\top\mY_k)$.
\end{enumerate}

\paragraph{Relationship with POGO.} As discussed in the main text, the last step in the algorithm above corresponds to the same normal update as \ours for the case when we set $\lambda = 1/2$, despite coming from completely different angles.
Moreover, if we assume that we have a smooth problem $r = 0$, the authors point out that then the Lagrangian $\Lambda$ has a explicit form, $\Lambda(\mX) = \Sym(\mX^\top \nabla f(\mX))$. Now, if we look at the proximal problem for the case when $r = 0$, it turns out that it looks like the following:
\begin{equation}
    \min_{\mD \in \sR^{n\times p} } \innerp{\mG}{\mD} + \frac{1}{2\eta}\norm{\mD - \mX}^2    
\end{equation}
Then, the derivative of the optimization function \wrt $\mD$ is of the form $\mG + \frac{1}{\eta} (\mD - \mX)$ which we can see easily that is zero when $\mD = \mX - \eta \mG$, and that is therefore the minimum for the proximal problem for this case.
Finally, if we plug $\Lambda$ inside the solution of the proximal problem we get that $\mY_k = \mX_k - \eta (\nabla f(\mX) - \mX_k \Sym(\mX^\top \nabla f(\mX)))$.

Therefore, if we set $\lambda = 1/2$ and \textit{not use any base optimizer}, we recover the normal step from \ours as well as a similar intermediate step, the difference being that we use $\mX\Skew(\mX^\top \nabla f(\mX))$, the Riemannian gradient under the canonical metric, and SLPG uses $\nabla f(\mX) - \mX_k \Sym(\mX^\top \nabla f(\mX))$, the Riemannian gradient under the Euclidean metric.
Both are quite similar, and indeed coincide when $p=1$ or $p=n$. The main difference otherwise is that the one that \ours uses is orthogonal to the normal direction, while the one used by SLPG contains an extra component which can drift the gradient update outside the tangent space.

Despite the similarities in methodology, it is worth noting that both papers start from quite different premises and make a different number of assumptions and theoretical contributions. For example, in this work we do not consider non-smooth regularization, but we provide a tighter bound for the distance of the updates, milder assumptions on the learning rate, as well as the introduction of unconstrained optimizers into the orthoptimizer. 
To the best of our knowledge, we are the first to implement SLPG besides the original authors, and the first to derive the simplified version of SLPG for which we can recover \ours for certain cases.

\section{Experimental Details and Results}
\label{app:sec:experiments}

In this section we provide additional experimental details in order to reproduce the experiments from the main section. All the experiments in the paper were produced under the same condition using a small cluster with 8 NVIDIA RTX A6000 GPUs.
We will make sure to release our code to easily reproduce our experiments upon acceptance.

For every plot showing time we use linear interpolation to sample at the same time steps for every independent run, taking then the average and, when it did not clutter too much the image as in \cref{fig:pca-procrustes}, plot a confidence interval of 90 for performances and 70 for distances (due to the log-scale producing artifacts). Unless otherwise specified, we use the default hyperparameter for all methods. Except for the PCA and Procrustes experiments (which we explain below), hyperparameters were search with a grid search of equal budget for all baselines, picking the set of hyperparameters that obtained best validation performance.

\subsection{Online PCA and Orthogonal Procrustes}

For these two experiments we adopted the codebase of RSDM which is publicly available at \url{https://github.com/andyjm3/RSDM}. Then, we simply had to introduce each of the baselines as drop-in replacements for the previously used optimizers. We additionally modified the code to log distances at each step, as well as to fix the randomness for each experiment.
Given that running each of these experiments is a matter of a few seconds, we manually tuned the hyperparameters of each individual orthoptimizer until we could not get any better results. 

In particular, we used for PCA $p=1500$, $n=2000$, set the dimension of RSDM to $700$ and then set the learning rate of each method as follows: $0.15$ for RGD, $1.5$ for RSDM, $0.25$ for Landing and \ours, $10.5$ for LandingPC, and $0.125$ for SLPG. We used SGD as base optimizer for \ours with momentum of $0.3$, set the momentum for Landing to $0.1$ and the $\lambda$ parameter of LandingPC to $0.01$ (for Landing we kept the default value of $1$). Then, we trained with each method 10 independent times and aggregated the results to take the average, stopping if we reached an optimality gap of \num{1e-6}.

Similarly, for Procrustes we set $n = p = 2000$ with a dimension for RSDM of $900$. Then we employed the following learning rates: $0.5$ for RGD, $2$ for RSDM, $0.5$ for Landing and \ours, $1.5$ for LandingPC, and $0.5$ for SLPG. For LandingPC we set $\lambda$ to $0.1$ this time, and reduced the momentum for the SGD base optimizer of \ours to $0.1$, sharing the same momentum as for Landing.

\subsection{Vision Transfomer}

Similar to PCA and Procrustes, we adapted the ViT experiment from the repository of RSDM located at \url{https://github.com/andyjm3/RSDM}.
In this case, we ran every experiment five independent times and trained for $10$ epochs. We kept the dimension of $300$ for the submanifold of RSDM as in the original code, and use the following learning rates after performing a grid search: $0.1$ for RGD, $0.5$ for RSDM, $0.001$ for Landing, $0.01$ for LandingPC and SLPG, as well as for \ours which also used SGD as base optimizer. Additionally, we set the momentum of Landing to $0.1$.

\subsection{Convolutional Neural Network}

For the experiments with the CNN, we adapted the code for the CIFAR-10 speedrun benchmarks, which is publicly available here: \url{https://github.com/KellerJordan/cifar10-airbench/}.
In terms of changes, we had to make sure that the initial parameters were orthogonal, for which we simply projected to the Stiefel manifold at initialization, as well as modified the code to impose the orthogonal constraints to the filters and kernels, depending of the experiment. We increased the default number of epochs from $20$ to $100$ and repeated all the experiments $5$ times.

After a grid search with equal budget, these are the hyperparameters that we used: 
\begin{enumerate}
    \item Orthogonal filters: learning rates of $0.01$ for RGD and Adam, $0.1$ for RSDM (with a submanifold dimension of $64$), $0.001$ for SLPG and Landing (the latter with a momentum of $0.6$), and we used a learning rate of $0.5$ for LandingPC and \ours. Additionally, we use VAdam as base optimizer for \ours.
    
    \item Orthogonal kernels: learning rates of $0.01$ for RGD,  Adam and Landing (this time with no momentum), $0.5$ for RSDM (with a submanifold dimension of $2$ out of $3$), $0.1$ for SLPG, and $0.5$ again for LandingPC and \ours (using VAdam as base optimizer).
\end{enumerate}

\subsection{Squared Unitary Probabilistic Circuits}

We consider probabilistic circuits (PCs) \citep{choi2020pc,vergari2021compositional}
that are representing a complex wave function and then a probability distribution once squared \citep{loconte2024subtractive,loconte2025sum}.
For our experiments, these squared PCs  are parameterized with complex unitary matrices.
We could not find any available code online for \citet{loconte2025square}. After contacting the authors through email, they provided access to their codebase to reproduce the experiments in their work. Therefore, we refer the readers to the experimental details within \citep{loconte2025square} for details on how to reproduce those experiments, as we did not change anything other than plugging in our baseline orthoptimizers.

In particular, the MNIST experiment we show in this article corresponds to the squared unitary probabilistic circuit with Kronecker products and $10$ units. The only difference with respect to their experimental setting is the halving of learning rate after finding a plateau, which we noticed it improved the results of all baselines.

Regarding the hyperparameter used for the optimizers, these are the values that we found to work the best given an equal budget between methods: learning rate of $0.05$ for RGD and LandingPC (the latter with a value $\lambda$ of $0.1$ in addition), $0.01$ for Landing and $0.5$ for \ours using VAdam as base optimizer. With respect to SLPG, we found that our considered learning rates lead to the model diverging within the first epoch, and so we found the best learning rate to be $0.0005$ ($0.0075$ already diverged). In a similar note, we could not make work RSDM with a submanifold size of $8$ for every value of learning rate that we tried: While the method did run, its performance was subpar to the rest of methods, not even lying within the range of the plots shown in the main paper.

\subsection{Ablation on Tensor Precision}
\label{app:subsec:ablation-pca}

In order to understand the reason why RSDM was monotonically increasing the distance of its iterates from the manifold, we decided to run a small experiment on the online PCA setting.

In this experiment, we took a subset of orthoptimizers to have as a reference for RSDM, and then reproduce the experiments from the main paper with two different settings. In the first, we activate the use of 16-bit precision floating points to compute matrix multiplications (still storing them as 32-bit precision numbers). This option sacrifices accuracy in the matrix multiplication by significant gains in running times. The second experiment instead ran all the experiment with 64-bit precision floating points, which significantly slowed the experiments in exchange of more accurate computations. It is important to remark that to track the results (optimality gap and manifold distance) we went back to the original setting at each iteration.

The results are summarized in in \cref{app:fig:ablation-pca-procrusts}, where we can make a couple of observations. Regarding training time, we see that \ours and Landing benefit the most from speed-ups in matrix multiplication and, conversely, all methods significantly increased training times (from $20$ to $500$ seconds) when using 64-bit floating point precision.
Regarding manifold distances, we first observe that all considered methods converge to the manifold for the highest precision, \textit{including RSDM}. This points to numerical problems within the computations of RSDM. Similarly, we see that the use of faster matrix multiplication come at the cost of less precision on the feasibility of the solutions, with all distances increasing order several orders of magnitude with respect to those results from \cref{fig:pca-procrustes}.

\begin{figure}[t]
    \centering
    \includegraphics[width=0.5\linewidth]{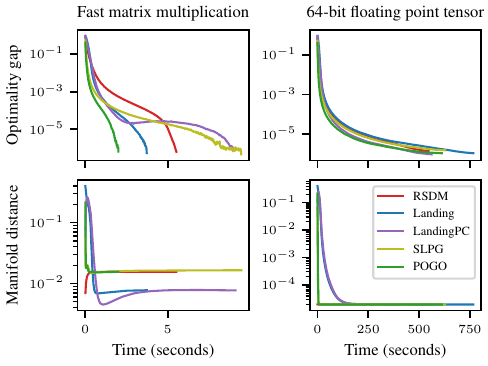}
    \caption{Ablation on online PCA with a subset of methods exploring the trade-off between multiplication speed and numerical precision.}
    \label{app:fig:ablation-pca-procrusts}
\end{figure}

\subsection{Ablation on Computing the Step Size}
\label{app:subsec:ablation}

In order to clearly show the advantage that \ours has from leveraging unconstrained optimizers, as well as to empirically demonstrate the theoretical results discussed in \cref{subsec:compute-lambda,subsec:approximate-lambda}, we ran an ablation study on this squared unitary PC experiment (\cref{subsec:osos}), where we take \ours with no base optimizer ($\mG = \nabla f(\mX)$) and vary the learning rate for $\eta\in\{0.001, 0.005, 0.007, 0.010,0.025\}$ for the case where we compute $\lambda$ by solving the landing polynomial (see \cref{subsec:compute-lambda}) or by setting $\lambda$ to $1/2$\eqp{.}

We present the results for each of these two options in \cref{fig:osos-ablation}, where we also plot for reference \ours with VAdam as base optimizer (the one selected for the experiments in \cref{subsec:osos}).
First, we observe that the version with VAdam obtains the best bpd results than any of the other hyperparameter settings, showing the competitive advantage that unconstrained optimizers can provide to \ours. 
Second, we see that when we compute $\lambda$ and increase the learning rate, \ours goes from staying in the manifold all the time to fluctuate more, to the point where it directly escapes from the manifold at $\eta = 0.025$. 
For the case where we fix $\lambda$ to $0.5$, we see that only two learning rates appear in \cref{fig:osos-ablation} (right). The reason is that every other version not appearing in the plot diverged within the first epoch, showing that $\lambda=0.5$ is indeed an approximation that works when we keep $\xi$ under control, and also that if we compute $\lambda$ we can use higher learning rates. 

Finally, we show in \cref{fig:osos-ablation2} the same models, where we put together the four versions of \ours with SGD sharing learning rate. In this figure, we can observe that there is no difference at all (besides training time) between fixing $\lambda$ or computing the root for the smallest learning rate (although the downstream performance is worse). We see a similar trend for $\lambda = 0.005$ for the downstream performance, where both versions of \ours are indistinguishable. However, we observe that in the distance to the manifold fixing $\lambda$ seems to converge faster. We foresee two possible explanations. First, it could be that our choice for picking the root is not the best one. Second, it is possible that at the beginning of training $\lambda = 0.5$ overshoots the correct root, alternating between both orientations of the manifold at each iteration.

However, these results empirically show that we can control how tight \ours stay on the manifold by reducing the learning rate. In practice, and as we show every time we use VAdam, the gradient normalization that it performs internally helps us to adaptively control $\norm{\mG}$, staying always really close to the manifold and allowing for larger learning rates.

\begin{figure}[t]
    \centering
    \begin{minipage}[t]{.5\linewidth}
        \centering
        \includegraphics[width=\linewidth]{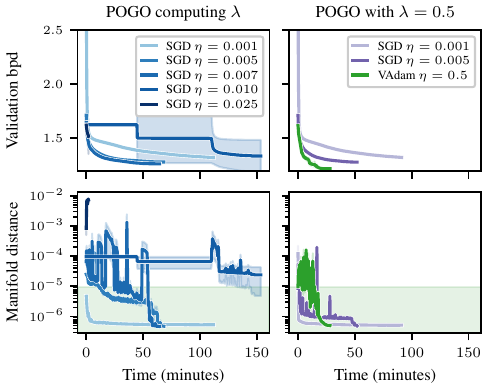}
        \caption{Ablation of \ours on the PC experiment as we change the learning rate with no base optimizer. Left plots solve the landing polynomial to compute $\lambda$ and right plots fix $\lambda$ to $1/2$. \Ours with VAdam \citep{ling2022vectoradam} is added as a reference.} \label{fig:osos-ablation}
    \end{minipage}%
    \hspace{2em}%
    \begin{minipage}[t]{.425\linewidth}
        \centering
        \includegraphics[width=.9\linewidth]{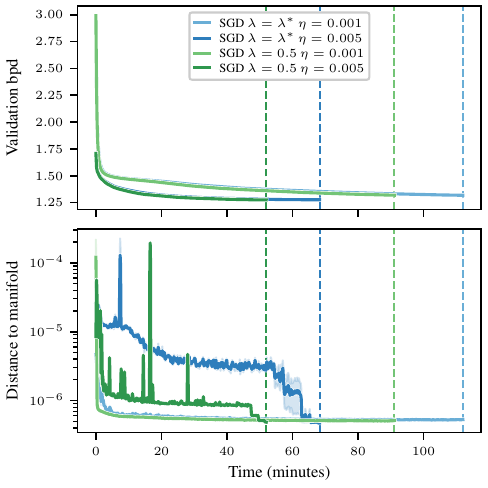}
        \caption{Same ablation as in \cref{fig:osos-ablation}, plotting together \ours solving the landing and polynomial and \ours with $\lambda = 1/2$ for $\eta = 0.001$ and $\eta = 0.005$. We see that both approaches are identical for the smallest step size, while they differ at the largest one, likely from $\lambda = 1/2$ over/under-estimating the value of $\lambda^*$.} \label{fig:osos-ablation2}
    \end{minipage}%
\end{figure}

\end{document}